\documentclass[journal]{IEEEtran}
\usepackage{amsmath,amsfonts}
\usepackage{algorithmic}
\usepackage{algorithm}
\usepackage{array}
\usepackage[caption=false,font=normalsize,labelfont=sf,textfont=sf]{subfig}
\usepackage{textcomp}
\usepackage{stfloats}
\usepackage{url}
\usepackage[hidelinks]{hyperref}
\usepackage{verbatim}
\usepackage{graphicx}

\usepackage{amsmath,amssymb}
\usepackage{amsthm}

\theoremstyle{plain}
\newtheorem{theorem}{Theorem}
\newtheorem{lemma}{Lemma}
\newtheorem{proposition}{Proposition}
\newtheorem{corollary}{Corollary}

\theoremstyle{definition}
\newtheorem{definition}{Definition}
\newtheorem{assumption}{Assumption}

\theoremstyle{remark}

\usepackage{cite}
\begin{document}

\title{SCGFM-ART: Amortized Relational Transport for Structure-Centric Graph Foundation Models}

\author{Xiaodong He, Xincheng Wang, Zhao Kang%
\thanks{X. He, X. Wang, and Z. Kang are with the University of Electronic Science and Technology of China, Chengdu, China.
This work has been submitted to the IEEE for possible publication. Copyright may be transferred without notice, after which this version may no longer be accessible.
E-mail: hexiaodong24@126.com; Zkang@uestc.edu.cn.
The source code is available at:
\url{https://github.com/Xd-He/SCGFM-ART}.}%
}

\markboth{}{}

\maketitle

\begin{abstract}
Graph foundation models (GFMs) aim to learn transferable representations across severely heterogeneous graph domains. However, severe domain shifts in topology, graph scale, and feature semantics impede the construction of a unified, domain-agnostic representation space. To address this, we propose SCGFM-ART, a structure-centric GFM framework that aligns arbitrary graphs onto a shared relational atlas via Amortized Relational Transport (ART). The relational atlas serves as a universal coordinate system defined by a finite set of relational landmarks (bases), while ART directly predicts reusable, end-to-end graph-to-base transport plans, bypassing costly runtime Gromov-Wasserstein optimizations. Under this formulation, SCGFM-ART decomposes a graph into a unified representation: globally via its relational response coordinates relative to the atlas, and locally via its node-to-role structural correspondences. These correspondences project disparate node attributes into a canonical role space, resolving structural and semantic heterogeneity within a singular alignment interface. Rigorously modeling graphs and atlas bases as finite measured relational spaces, we establish coordinate fidelity bounds, prove stability under predicted transport plans, and derive an amortized coverage bound that guarantees our learning objective tightly surrogates ideal relational coverage. Benchmarked across 14 cross-domain graph- and node-level classification tasks, SCGFM-ART achieves state-of-the-art transferability, securing superior average ranks of 2.29 and 1.14, respectively. Topological perturbation analyses demonstrate that node-role transport retains fine-grained structural nuances beyond global coordinates. On real-world benchmarks, the amortized formulation yields 44.2×–85.1× faster frozen target-domain inference by avoiding iterative alignment at test time.
\end{abstract}


\section{Introduction}
\IEEEPARstart{F}{oundation} models have revolutionized representation learning across natural language processing and computer vision by leveraging large-scale pretraining to extract highly transferable representations~\cite{devlin2019bert,brown2020gpt3,dosovitskiy2021vit,he2022mae}.
This paradigm is rapidly extending to graph-structured domains, where Graph Foundation Models (GFMs) seek to capture domain-agnostic, reusable structural knowledge across diverse datasets and downstream tasks~\cite{mao2024gfm,wang2025git,he2026came,hu2026gtalign}.
To this end, recent advances have explored unified task interfaces, multi-domain topological alignment, transferable structural vocabularies, and theoretical bounds on cross-domain knowledge transfer~\cite{wang2025git,wang2025mdgfm,yuan2025bridge,sun2025riemanngfm}.

A central obstacle in building general-purpose GFMs stems from the intrinsic non-Euclidean heterogeneity across graph domains. Graphs originating from distinct environments lack a universal representation space: their structures vary dramatically in scale, edge density, and relational topology, while node attributes frequently differ in dimensionality, semantics, or availability. 
Consequently, cross-domain graph pretraining demands a unified mechanism capable of reconciling both \emph{structural heterogeneity} and \emph{feature heterogeneity}.
A robust GFM must therefore establish a canonical reference interface through which topological structures can be directly compared and node attributes can be consistently aligned.

Recent graph tokenization paradigms address this challenge by discretizing nodes, edges, or substructures into reusable structural tokens~\cite{kim2022tokengt,wang2025gqt,chen2025hight}.
While this offers a shared vocabulary across architectures, relational comparability across graphs of varying scales and topological regimes remains constrained by underlying structural shifts.
This raises a fundamental question: \emph{Can heterogeneous graphs be embedded into a shared relational reference system that simultaneously acts as a universal coordinate space for topology and a local alignment interface for node attributes?}

We address this question through a geometric perspective. 
If a finite collection of universal relational structures can serve as shared \emph{landmarks}, any arbitrary graph can be uniquely characterized by its geometric position relative to these landmarks—analogous to positioning an object within a global coordinate frame. 
Once a common set of landmarks is established, graphs with disparate cardinalities and domain-specific topologies can be projected into identical coordinate dimensions, establishing a shared structural interface for foundation-level representation learning.

Gromov--Wasserstein (GW) geometry provides a mathematically rigorous foundation for this formulation. 
By evaluating objects through their internal pairwise relational structures~\cite{memoli2011gw}, and extending to weighted networks~\cite{chowdhury2019networkgw}, GW alignment offers a metric space for comparing arbitrary relational geometries. 
Furthermore, optimal transport couplings under GW distances yield meaningful node-level correspondences for cross-domain matching and representation transfer~\cite{xu2019gwl}.
These properties suggest that heterogeneous graphs can be effectively characterized by their relational distances and transport couplings relative to a finite set of shared reference bases.

\begin{figure}[t]
\centering
\includegraphics[width=\columnwidth]{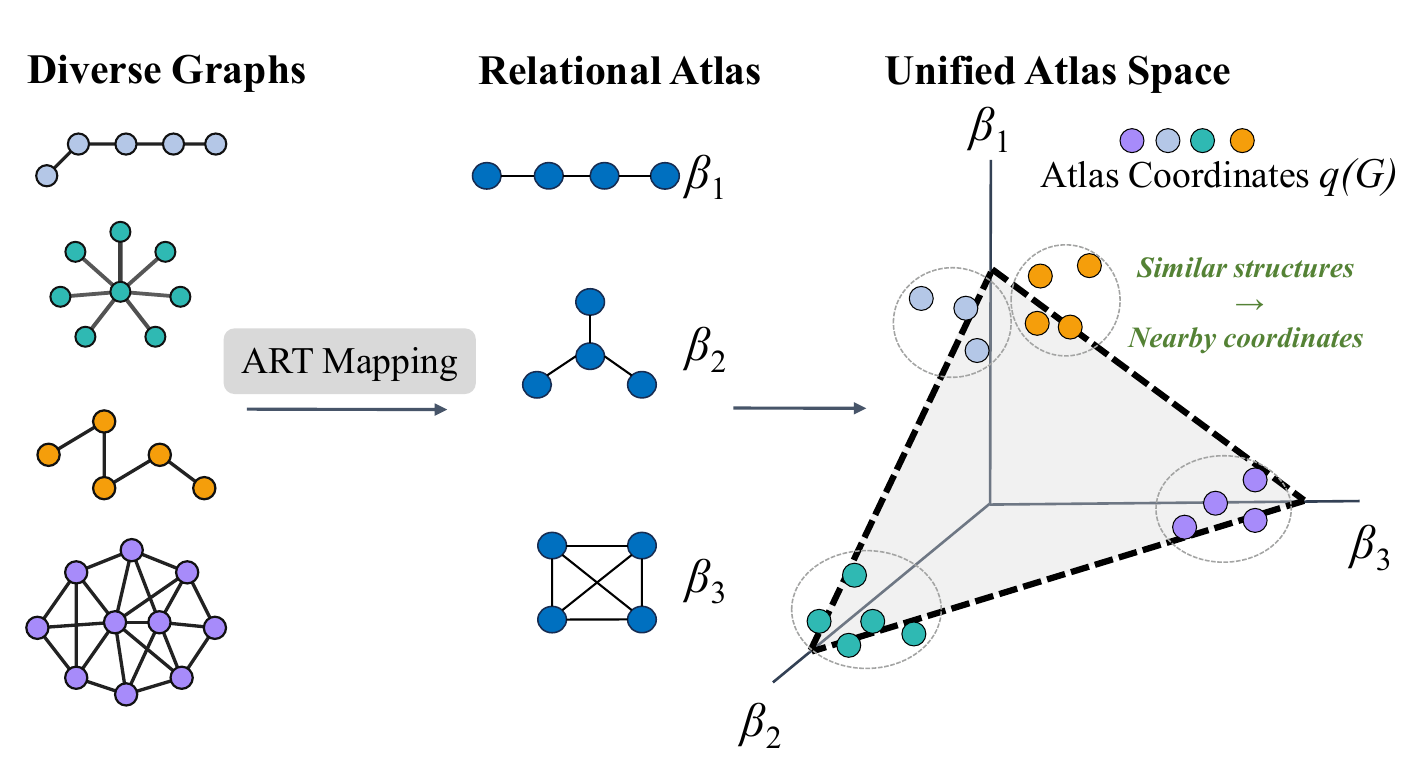}
\caption{
\textbf{Geometric intuition of SCGFM-ART.} Heterogeneous input graphs from diverse domains (left) with disparate scales and node semantics are mapped to a shared \textbf{Relational Atlas} ($\beta_1, \beta_2, \dots, \beta_K$) via \textbf{Amortized Relational Transport (ART)}. Graph--base relational discrepancies induce global response coordinates $q(G)$ in a canonical structural reference space (right), positioning topologically similar graphs nearby regardless of original domain or size.
}
\label{fig:motivation}
\end{figure}

As illustrated in Fig.~\ref{fig:motivation}, we organize these shared references into a \emph{relational atlas}. 
Under this construction, each graph is represented \emph{globally} by its relational discrepancies to the atlas bases—forming its global atlas coordinate $q(G)$—and \emph{locally} by the corresponding graph--base transport couplings, which align graph nodes with shared relational roles. 
This yields a dual global--local interface that seamlessly addresses structural and feature heterogeneity.

Building upon this formulation, we introduce \textbf{SCGFM-ART}, a structure-centric graph foundation model anchored by a shared relational atlas. 
The atlas provides a finite system of topological landmarks, while Amortized Relational Transport (ART) learns an end-to-end mapping from input graphs directly to these landmarks.
For any graph, ART produces both a global relational response coordinate and a fine-grained node-to-role structural mapping. 
The global coordinates project graphs into a canonical reference frame regardless of scale or node identity, whereas the local transport couplings map domain-specific node attributes into a canonical role space, resolving feature heterogeneity within the same alignment pipeline. 
By amortizing the optimal transport optimization via neural inference~\cite{amos2023metaot}, ART bypasses costly runtime iterations, producing instantaneous graph--atlas alignment for unseen target graphs.

We establish rigorous theoretical foundations for SCGFM-ART by modeling graphs and atlas bases as finite measured relational spaces. 
Assuming total boundedness of the task support, we prove that a finite relational atlas establishes a well-defined coordinate system with geometric distortion bounded by the atlas covering radius. 
We further quantify the excess error introduced by neural transport amortization and derive an amortized coverage bound, proving that our practical pretraining objective serves as a provably sound surrogate for ideal relational atlas coverage.

Our main contributions are summarized as follows:
\begin{enumerate}
    \item \textbf{A Unified Dual-Level Framework:} We introduce \textbf{SCGFM-ART}, which maps heterogeneous graphs onto a shared relational atlas. Through ART, it unifies global relational coordinates and local node-to-role correspondences into a single interface for structural and semantic alignment.
    \item \textbf{Theoretical Foundations:} We establish formal guarantees for finite relational atlas learning, establishing coordinate fidelity under finite atlas coverage, quantifying amortization prediction error, and proving that our trainable coverage objective tightly bounds ideal relational alignment.
    \item \textbf{Extensive Empirical Validation:} Across 14 cross-domain graph- and node-level benchmarks, SCGFM-ART demonstrates state-of-the-art transferability. Comprehensive empirical studies validate component efficacy, atlas capacity, structural perturbation sensitivity, and computational scalability.
\end{enumerate}

\textbf{Relationship to Prior Conference Work.} A preliminary version of this concept appeared as SCGFM~\cite{he2026scgfm}, which introduced learnable geometric bases as a shared structural reference system. However, SCGFM lacked a predictive mapping path, requiring computationally expensive, instance-wise Gromov--Wasserstein optimization for every unseen target graph during inference. This work overcomes this bottleneck by proposing \textbf{ART}, which directly predicts reusable graph--base transport plans. By eliminating iterative target-time alignment, ART accelerates frozen-inference throughput by $44.2\times\text{--} 85.1\times$ over SCGFM on real-world benchmarks. Furthermore, we establish a comprehensive relational-atlas theoretical framework and substantially expand our empirical evaluations with additional baselines, topological perturbation experiments, and scalability analyses.

\section{Related Work}

\subsection{Self-Supervised Graph Pre-training}

Self-supervised graph learning exploits unlabeled graph data to learn transferable representations. Contrastive methods such as DGI~\cite{velickovic2019dgi} and GraphCL~\cite{you2020graphcl} maximize agreement between related graph views or local--global representations, while masked modeling methods such as GraphMAE~\cite{hou2022graphmae} reconstruct corrupted node attributes.
FUG~\cite{zhao2024fug} extends contrastive pre-training across datasets
by accommodating node features of different dimensionalities without
rebuilding the graph encoder.
These approaches have established effective pre-training paradigms for graph representation learning. Graph foundation models further extend this objective from individual datasets to heterogeneous graph domains, where transferable representations must accommodate substantial variations in topology and node attributes.

\subsection{Graph Foundation Models}
GFMs aim to acquire reusable knowledge across graph domains and downstream tasks~\cite{liu2025gfm}.
Recent approaches explore different mechanisms for establishing shared representations.
GIT~\cite{wang2025git} unifies node-, edge-, and graph-level tasks through task-trees, while RiemannGFM~\cite{sun2025riemanngfm} constructs a structural vocabulary of trees and cycles in Riemannian spaces.
For multi-domain transfer, SAMGPT~\cite{yu2025samgpt} introduces structure tokens and prompting, MDGFM~\cite{wang2025mdgfm} aligns cross-domain topologies, and BRIDGE~\cite{yuan2025bridge} learns domain-invariant representations with selective knowledge transfer.
For robust downstream adaptation, GRAVER~\cite{yuan2025graver} constructs generative graph vocabularies from transferable subgraph patterns, while RAG-GFM~\cite{yuan2026raggfm} externalizes semantic andstructural knowledge through retrieval-augmented generation.
Complementing these multi-domain alignment strategies, transfer-invariant feature modeling~\cite{zhao2026transferinvariant} maps node attributes
with heterogeneous dimensions into a shared structural space, providing a unified feature interface for general graphs.
SCGFM~\cite{he2026scgfm} introduced learnable geometric bases and represented heterogeneous graphs by their GW discrepancies to a shared structural reference system.
Building on this structure-centric view, the present work learns a relational atlas together with graph-to-base correspondences, enabling each base to serve simultaneously as a global structural landmark and a local relational reference.

\subsection{Gromov--Wasserstein Learning and Amortized Transport}

Optimal transport (OT) compares probability measures through transport
couplings, with entropic regularization enabling efficient Sinkhorn
optimization~\cite{cuturi2013sinkhorn}. 
GW distance extends this principle to objects described by intrinsic
pairwise relations~\cite{memoli2011gw}, and has been generalized to
weighted networks~\cite{chowdhury2019networkgw}.
GW-based methods have been applied to structured graph comparison through Fused
GW~\cite{vayer2019fgw} and to joint graph matching and node embedding~\cite{xu2019gwl}. These formulations typically optimize transport separately for individual problem instances.
Meta OT~\cite{amos2023metaot} instead learns to amortize repeated OT problems across input measures. ART adopts this amortization principle for relational graph matching: it predicts feasible graph-to-base couplings whose relational energies define atlas coordinates, while the same couplings support downstream feature recoding.

\section{Preliminaries}
\textbf{Problem Setup.} Let $G \sim \mathbb P$ denote a graph drawn from the task domain. Each graph provides a finite node set $V$, an internal pairwise relation, and optional node attributes $X \in \mathbb{R}^{N \times d_f}$. Throughout, $N =\vert V \vert$ denotes graph size, and $\beta_k$ denotes the $k$-th relational geometric base, where $K$ bases and $M$ roles per base. We use $B_k$ for the relation matrix of $\beta_k$, a superscript $*$ for an exact relational optimum, and a hat for an amortized prediction.

\subsection{Relational problem formulation}
\begin{definition}[Finite measured relational structure] 
\label{def:relation_structure}
A finite measured relational structure is a triplet $G = (V, A, \mu_G)$, where $A: V \times V \rightarrow{[0,1]}$ is a symmetric relation satisfying $A(i,i) =0$, and $\mu_G$ is a probability measure on $V$. Finite metric-measure spaces arise when $A$ additionally satisfies the triangle inequality; normalized adjacency and kernel relations are included without requiring that condition.
\end{definition}
For graph $G$, we use the degree-smoothed measure
\begin{equation}
    \mu_G (i) = \frac{\mathrm{deg}(i)+ \epsilon_u}{\sum_{j \in V} [\mathrm{deg}(j) +\epsilon_u]},
\end{equation}
where $\mathrm{deg}(i)$ denotes the degree of node $i \in V$, and $\epsilon_u$ is a small positive constant for numerical stability.
Each relational geometric base uses the uniform measure
\begin{equation}
\beta_k=([M],B_k,\nu),
\qquad
[M]:=\{1,\ldots,M\},
\end{equation}
where $B_k\in [0,1]^{M\times M}$ is a symmetric relation matrix satisfying $B_k(a,a)=0$, and $\nu_a =\frac{1}{M}$, $a\in [M]$.
$G$ and $\beta_k$ may have different cardinalities; their comparison depends only on their internal relations and probability measures. For two measured relational structures $G=(V,A,\mu_G)$ and $G' = (V',A',\mu_{G'})$, let 
\begin{equation}
    \Pi(\mu_G,\mu_{G'}) = \{  T\in \mathbb R_+^{\vert V \vert \times \vert V'\vert }: T\boldsymbol{1}=\mu_G, T^{\top}\boldsymbol{1}=\mu_{G'}\}
\end{equation}
be the set of coupling with prescribed marginals. For any $T\in \Pi(\mu_G,\mu_{G'})$, define its squared relational distortion as 
\begin{equation}
\label{eq:relational_distortion}
    \mathcal E(A,A',T) = \sum_{i,j,m,n}(A_{ij} -A'_{m,n})^2 T_{im}T_{jn}.
\end{equation}
The corresponding relational GW discrepancy is
\begin{equation}
\label{eq:gw_discrepancy}
    d(G,G') = \sqrt{\min_{T\in\Pi(\mu_G,\mu_{G'})}\mathcal E(A,A';T)}.
\end{equation}

\begin{proposition}[Relational graph geometry]
Let $d$ be the relational GW discrepancy defined in Eq.\eqref{eq:gw_discrepancy}.
Then $d$ is finite and invariant under node relabeling. Modulo zero-discrepancy equivalence, it defines a metric on $\mathcal X$, which contains both input graphs and relational bases.
\end{proposition}

\subsection{Task-domain premise and finite atlas}
\begin{assumption}[Total boundedness of the task support]
Let $\mathcal X_0=\operatorname{supp}(\mathbb P)\subseteq \mathcal X$ be the support of the task graph distribution, equipped with the relational GW metric $d$ defined in Eq.\eqref{eq:gw_discrepancy}.
We assume that $(\mathcal X_0,d)$ is totally bounded. Equivalently, at every resolution, $\mathcal X_0$ can be covered by finitely many relational neighborhoods.
\end{assumption}

\begin{definition}[Finite relational atlas]
\label{def:atlas}
    A K-base relational atlas is a finite collection
    \begin{equation}
        \mathcal A =\{ \beta_1, \ldots,\beta_K\} \subseteq \mathcal{X}.
    \end{equation}
    Its covering radius over $\mathcal{X}_0$ and its induced ideal coordinate are $\epsilon_\mathcal{A} := \sup_{G \in \mathcal X_0} \min_{1 \leq k \leq K} d(G, \beta_k),$ $ q^*_{\mathcal A}(G):= \left(d(G,\beta_k)\right)_{k=1}^K$.
    We write $q^*$ when $\mathcal A$ is clear from context.
\end{definition}
\begin{theorem}[Finite-atlas coordinate fidelity]
    \label{thm:atlas_fidelity}
    Let $\mathcal A$ be any finite atlas with covering radius $\epsilon_\mathcal{A}$ over $\mathcal X_0$. Then, for any $G,G' \in \mathcal X_0$,
    \begin{equation}
        d(G,G')-2\epsilon_\mathcal{A} \leq \Vert q^*_{\mathcal A}(G) - q^*_{\mathcal A}(G') \Vert_{\infty} \leq d(G,G')
    \end{equation}
\end{theorem}
Theorem~\ref{thm:atlas_fidelity} establishes the geometric role of a
finite relational atlas. The ideal coordinate $q_{\mathcal A}^*(G)$
is non-expansive with respect to the relational discrepancy, while its
pairwise information loss is controlled by the realized atlas radius
$\epsilon_{\mathcal A}$.
Thus, a sufficiently representative finite atlas
provides a shared structural coordinate system for heterogeneous graphs.
All proofs and additional theoretical details are provided in the supplementary material.

The above analysis considers ideal coordinates derived from exact relational matching. In practice, however, solving a separate coupling optimization for every graph–base pair is computationally prohibitive. SCGFM-ART therefore adopts an amortized coupling predictor to approximate the optimal alignment efficiently.

\section{Method}
\begin{figure*}[!t]
\centering
\includegraphics[width=\linewidth]{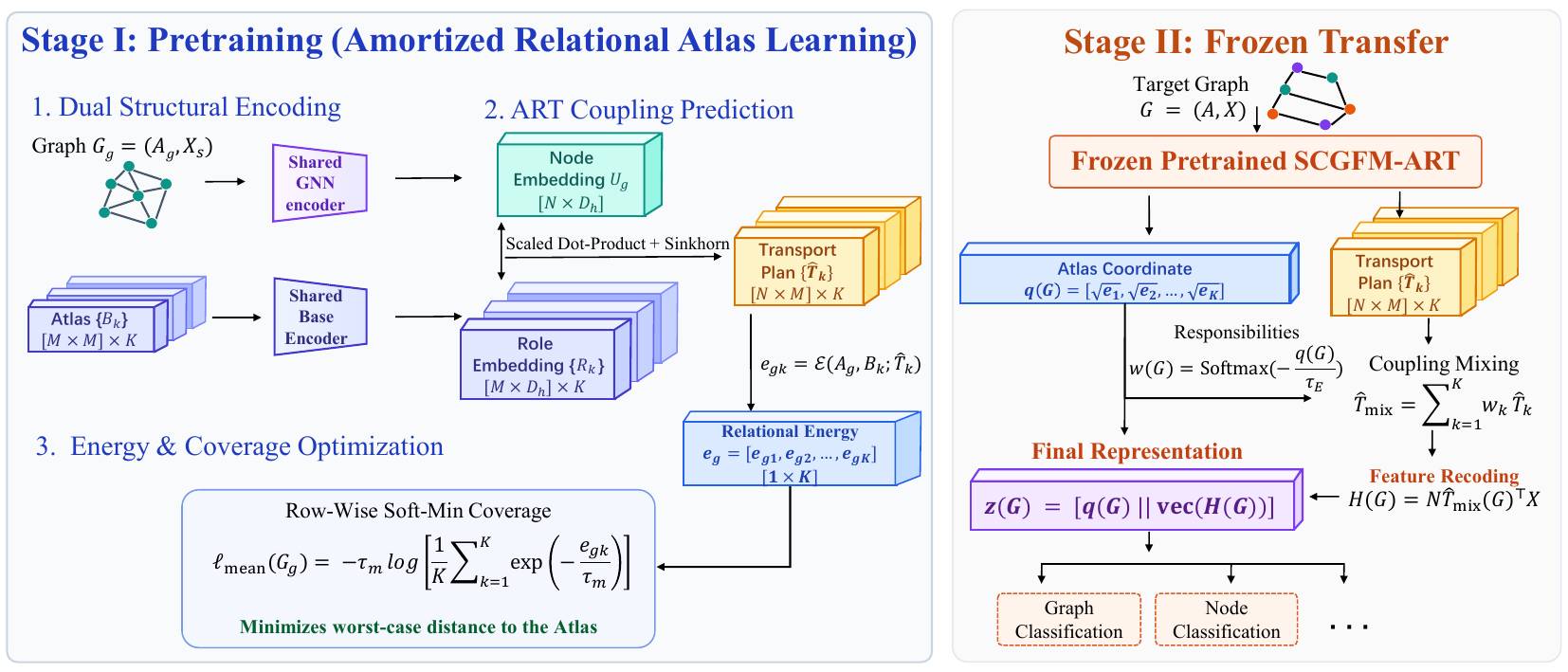}
\caption{Overview of SCGFM-ART.
\textbf{Stage I: Pretraining}. Given an input graph, a structure-only GNN encodes normalized degree features $X_s$, while a shared base encoder embeds the learnable relational atlas.
For each graph–base pair, ART predicts a feasible coupling through scaled dot-product and Sinkhorn projection, and evaluates the corresponding relational energy.
The resulting graph–base energy matrix is optimized by row-wise soft-min mean coverage to jointly learn the atlas and ART.
\textbf{Stage II: Frozen Transfer}.
The pretrained model directly produces target couplings and energies. Energy-derived responsibilities are used to mix graph–base couplings, while the square-root energies form the atlas coordinate.
Node attributes are transported to the shared base-role space via $H(G)=N\widehat T_{\mathrm{mix}}^{\top}X$, and the final representation $(z(G)=[q(G)\Vert\operatorname{vec}(H(G))])$ is used for downstream graph- and node-level tasks.}
\label{fig:framework}
\end{figure*}

\subsection{Overview}
We propose \textbf{SCGFM-ART}, a Structure-Centric Graph Foundation Model based on Amortized Relational Transport. As illustrated in Fig.~\ref{fig:framework}, SCGFM-ART learns shared relational atlas from unlabeled graph structures, predicts graph-to-base couplings via a ART module, and constructs transferable representations from the resulting relational coordinates and transport-conditioned features.

\subsection{Amortized Relational Transport}
The ART module amortizes relational matching across graphs and bases.
It takes structure-only graph embeddings and learned base-role embeddings as input, predicts all graph-to-base couplings through shared compatibility scores and Sinkhorn projection, and converts the resulting relational energies into atlas coordinates.
This feed-forward construction avoids per-sample GW outer optimization while preserving the prescribed graph and base marginals.
For observed graphs, $A$ is the loop-free binary adjacency relation and is stored as a coalesced edge list $E$ throughout the ART computation.

\textbf{Structure-only Graph Encoder.}
For a graph $G=(V,A,\mu_{G})$, SCGFM-ART adopts normalized node degree as the only structural node-wise feature:
\begin{equation}
    X_{s}(i) = \frac{\mathrm{deg}(i)}{\max_{j \in V} \mathrm{deg}(j)\vee 1}, \quad X_s \in \mathbb R^{N}
\end{equation}
A two-layer GIN~\cite{xu2019gin} shared across graph domains maps $X_s$ to node embeddings:
\begin{equation}
    U = \mathrm{GIN}(A,X_{s}) \in \mathbb R^{N \times D_h}
\end{equation}

SCGFM-ART instantiates the finite atlas in Definition \ref{def:atlas} with $K$ learnable kernel relational geometric bases.
For the $k$-th base, let $Z_k \in \mathbb R^{M\times M}$ be an unconstrained trainable parameter matrix.
Its entries are independently initialized from the standard normal distribution, $Z_k(a,b)\sim\mathcal N(0,1)$.
Let $\sigma(t)=(1+\exp(-t))^{-1}$ denote the element-wise sigmoid function, while
$\operatorname{Hollow}(S)$ sets the diagonal of $S$ to zero and leaves its
off-diagonal entries unchanged.
The relation matrix of the $k$-th base is parameterized as
\begin{equation}
    B_k=\operatorname{Hollow}\left[\sigma\left(\frac{Z_k+Z_k^\top}{2}\right)\right].
\end{equation}
Symmetrization, element-wise sigmoid, and diagonal removal make $B_k$ a
bounded, symmetric, hollow relation. Each base uses the uniform measure
$\nu$ specified in Definition \ref{def:relation_structure}.
Each $B_k$ is represented on the fixed index set $[M]$.
Its $a$-th row, $B_k(a,:)$, records the relations between base node $a$ and all
base nodes. A shared MLP is applied row-wise to every base matrix:
\begin{equation}
    R_k(a,:)=\operatorname{BaseEnc}\bigl(B_k(a,:)\bigr),\quad R_k\in\mathbb R^{M\times D_h}.
\end{equation}

Given the graph-node embeddings $U$ and the base-role embeddings $R_k$, we construct the graph-to-base compatibility logits using scaled dot-product similarity
\begin{equation}
    L_k =\frac{(UW_G)(R_kW_B)^{\top}}{\sqrt{D_h}} \in \mathbb{R}^{N\times M}
\end{equation}
where $W_G, W_B \in \mathbb R^{D_h \times D_h}$ are learnable linear projections shared across all graph--base pairs.
The proposal logits are subsequently projected onto the transport polytope $\Pi(\mu_G,\nu)$ using log-domain Sinkhorn normalization:
\begin{equation}
\widehat T_k=\operatorname{Sinkhorn}(L_k;\mu_G,\nu) \in \Pi(\mu_G,\nu).
\label{eq:art_sinkhorn}
\end{equation}
ART employs a single cross-transport stage. The shared encoders provide the two structural representations, scaled dot-product compatibility constructs the pairwise proposal, and Sinkhorn enforces the prescribed marginals.

\begin{lemma}[Feasibility and permutation equivariance of ART]
\label{lem:art_equivariance}
Assume that the proposal kernel and both marginals are strictly positive.
Then the Sinkhorn limit $\widehat T_k$ satisfies
\begin{equation}
\widehat T_k\mathbf{1}_M=\mu_G,
\qquad
\widehat T_k^\top\mathbf{1}_N=\nu.
\end{equation}
Moreover, for any node permutation matrix $P \in \{0,1\}^{N\times N}$,
\begin{equation}
\widehat T_k(PAP^\top,P\mu_G)=P\widehat T_k(A,\mu_G).
\end{equation}
\end{lemma}
Therefore, the resulting graph–base relational energy is invariant under node permutations.

For any feasible coupling $T$, the relational energy admits the compact form
\begin{equation}
\mathcal E(A,B;T)
=
\sum_{i,j} A_{ij}^{2}\mu_i\mu_j
+
\sum_{a,b} B_{ab}^{2}\nu_a\nu_b
-
2\langle AT,TB\rangle_F.
\label{eq:fast_energy}
\end{equation}
Eq.~\eqref{eq:fast_energy} is algebraically equivalent to the squared relational distortion defined in Eq.~\eqref{eq:relational_distortion}, replacing the explicit four-index summation with matrix operations.
For the sparse binary input relation, the graph-side terms can be evaluated exactly through edge-wise reductions, without materializing a dense $N\times N$ relation matrix.

ART replaces the exact graph-to-base optimization with the amortized
coupling $\widehat T_k$. For the $k$-th base, we define
\begin{align}
e_k^*(G)
&:=
\min_{T\in\Pi(\mu_G,\nu)}\mathcal E(A,B_k;T)= d^2(G,\beta_k),
\\
e_k(G)
&:=\mathcal E(A,B_k;\widehat T_k)\ge e_k^*(G),\qquad q_k(G):=\sqrt{e_k(G)}.
\end{align}
Since $\sqrt{e_k^*(G)}=d(G,\beta_k)$, the exact energies recover the
ideal atlas coordinate $q_{\mathcal A}^*(G)$ defined in
Definition~\ref{def:atlas}.
Let $q(G):=(q_k(G))_{k=1}^K$, $\eta(G):=\left\|q(G)-q_{\mathcal A}^*(G)\right\|_\infty$, where $\eta(G)$ measures the worst-case discrepancy between the amortized and ideal coordinates over the relational atlas.

\begin{corollary}[Amortized Atlas Coordinate Fidelity]
\label{cor:amortized_coordinate_fidelity}
Let $\mathcal A$ be a finite relational atlas with covering radius
$\epsilon_{\mathcal A}$ over $\mathcal X_0$. For the ART coordinate
$q(G)$ defined above, any $G,G'\in\mathcal X_0$ satisfy
\begin{equation}
\begin{aligned}
d(G,G')
-2\epsilon_{\mathcal A}
-\eta(G)-\eta(G')
&\le
\left\|q(G)-q(G')\right\|_\infty
\\
&\le
d(G,G')
+\eta(G)+\eta(G').
\end{aligned}
\label{eq:amortized_coordinate_bound}
\end{equation}
\end{corollary}

Corollary~\ref{cor:amortized_coordinate_fidelity} separates two sources of distortion in the practical atlas representation.
The term $\epsilon_{\mathcal A}$ reflects the finite resolution of the learned atlas, whereas $\eta(G)$ quantifies the additional coordinate error introduced by amortized transport.
Therefore, ART preserves the relational-coordinate fidelity of the ideal atlas up to the combined effects of atlas resolution and amortization error.

\subsection{Learning the Relational Atlas}
\label{sec:atlas_learning}
The preceding analysis characterizes the representation induced by a
given relational atlas and quantifies the additional distortion introduced
by amortized coupling prediction. It remains to determine how the atlas
and the ART predictor can be learned jointly from unlabeled graphs.
We address this problem through a relational coverage objective that
encourages each graph to admit at least one low-energy correspondence
to the shared atlas.

\textbf{Mean Relational Coverage}
\label{sec:mean_coverage}
For a graph $G_g$ in the current mini-batch and each relational base $\beta_k$, we define
$e_{gk}=\mathcal E(A_g,B_k;\widehat T_{gk}).$
The normalized soft-min coverage objective for $G_g$ is
\begin{equation}
\ell_{\mathrm{mean}}(G_g) =-\tau_m\log\left[\frac{1}{K}\sum_{k=1}^K\exp\left(-\frac{e_{gk}}{\tau_m}\right)\right],
\label{eq:mean_coverage_graph}
\end{equation}
where $\tau_m$ is a temperature parameter.
The corresponding mini-batch objective is
\begin{equation}
\mathcal L_{\mathrm{mean}}=\frac{1}{|\mathcal I_{\mathrm{bat}}|} \sum_{g\in\mathcal I_{\mathrm{bat}}} \ell_{\mathrm{mean}}(G_g),
\label{eq:mean_coverage_batch}
\end{equation}
where $\mathcal I_{\mathrm{bat}}$ is the set of graphs per mini-batch.
We optimize the atlas and ART parameters jointly using $\mathcal{L}_{\text{mean}}$.

The soft-min objective provides a differentiable approximation to nearest-base assignment while allowing all bases to receive gradients.
However, the energies used in Eq.~\eqref{eq:mean_coverage_graph} are evaluated at amortized couplings rather than exact relational optima.
We therefore analyze how closely the optimized objective reflects the ideal finite-atlas coverage.
\begin{theorem}[Amortized Coverage Bound]
\label{thm:amortized_coverage}
Define the ideal and amortized nearest-base energies as
\begin{equation}
m^*(G)=\min_k e_k^*(G),\qquad \widehat m(G)= \min_k e_k(G).
\label{eq:nearest_base_energy}
\end{equation}
Let $\mathcal K^*(G)=\arg\min_k e_k^*(G)$ denote the set of ideal nearest bases, and define the ART excess over these bases as
\begin{equation}
\Delta(G)=\min_{k\in\mathcal K^*(G)}\left[e_k(G)-e_k^*(G)\right]\ge 0.
\label{eq:art_excess}
\end{equation}
Because $\widehat T_k$ implies $e_k(G)\ge e_k^*(G)$, the normalized soft-min objective satisfies
\begin{equation}
m^*(G) \le \ell_{\mathrm{mean}}(G) \le m^*(G)+\Delta(G)+\tau_m\log K.
\label{eq:coverage_bound}
\end{equation}
\end{theorem}
Theorem~\ref{thm:amortized_coverage} directly connects the trainable objective to ideal finite-atlas coverage. The relaxation gap is determined by two terms: the graph-dependent coupling excess $\Delta(G)$ introduced by amortized transport and the temperature-dependent soft-min approximation term $\tau_m\log K$. 
\begin{corollary}[Finite-Atlas Coverage Consistency]
\label{cor:coverage_consistency}
Let $\epsilon_{\mathcal A}$ denote the covering radius of atlas $\mathcal A$ over the task support. For every $G\in\mathcal X_0$: $m^*(G)\le \epsilon_{\mathcal A}^2$.
Combining this relation with Theorem~\ref{thm:amortized_coverage} gives
\begin{equation}
0\le \ell_{\mathrm{mean}}(G) \le \epsilon_{\mathcal A}^2 +\Delta(G)+ \tau_m\log K.
\label{eq:coverage_consistency}
\end{equation}
\end{corollary}

Corollary~\ref{cor:coverage_consistency} closes the connection between the ideal finite-atlas geometry and the objective optimized in practice.
The achievable coverage is governed by three factors: the atlas radius $\epsilon_{\mathcal A}$, the amortized alignment excess $\Delta(G)$, and the soft-min relaxation $\tau_m\log K$.

\subsection{ART-full Representation for Transfer}
Together, Theorem~\ref{thm:atlas_fidelity}, Corollary~\ref{cor:amortized_coordinate_fidelity}, and Theorem~\ref{thm:amortized_coverage} connect the learned SCGFM-ART representation to the underlying relational geometry: the finite atlas provides stable structural coordinates, while ART approximates graph–base alignment with controlled error under a tractable coverage objective.

Beyond scalar relational energies, the learned couplings encode fine-grained node-to-role correspondences between each graph and the shared atlas. These correspondences are directly reused to construct transport-conditioned feature representations for downstream transfer.

\subsubsection{Base Responsibility and Coupling Mixture}
\label{sec:base_responsibility}

Given a target graph $G$, the frozen SCGFM-ART model produces the
graph-to-base energies $\{e_k(G)\}_{k=1}^K$, the corresponding atlas
coordinate $q(G)$, and the amortized couplings $\{\widehat T_k\}_{k=1}^K$.
We convert the relational energies into base responsibilities as
\begin{equation}
w_k(G)=\frac{\exp\left(-e_k(G)/\tau_E\right)}{\sum_{\ell=1}^{K} \exp\left(-e_\ell(G)/\tau_E\right)},
\label{eq:base_responsibility}
\end{equation}
where $\tau_E>0$ is the responsibility temperature.
The atlas coordinate $q(G)$ retains the graph-to-base relational responses, whereas $w(G)=(w_k(G))_{k=1}^K$ determines the contribution of individual bases to feature recoding.
Since all relational bases are defined on the common role index set
$[M]$, their couplings can be combined directly:
\begin{equation}
T_{\mathrm{mix}}(G)=\sum_{k=1}^{K} w_k(G)\widehat T_k \in \Pi(\mu_G,\nu).
\label{eq:coupling_mixture}
\end{equation}
$T_{\mathrm{mix}}(G)$ aggregates the corresponding node-to-role mappings for downstream feature recoding.

\subsubsection{Transport-Conditioned Feature Recoding}
\label{sec:feature_recoding}

The coupling mixture maps the node attributes $X$ into the common
base-role space:
\begin{equation}
H(G)
=
N T_{\mathrm{mix}}(G)^\top X
\in
\mathbb R^{M\times d_f},
\label{eq:feature_recoding}
\end{equation}
Graphs with different numbers of nodes are represented by a fixed number of $M$ relational roles while retaining their graph-specific node attributes.

\subsubsection{Final Graph Representation}
\label{sec:final_representation}

The final ART-full representation concatenates the relational atlas
coordinate with the vectorized transport-conditioned feature map:
\begin{equation}
z(G)=\left[q(G)\| \operatorname{vec}\bigl(H(G)\bigr)\right].
\label{eq:art_full_representation}
\end{equation}

ART-full is invariant to permutations of the input nodes. Specifically,
for any permutation matrix $P$ acting on $V$, let $G'=(V,PAP^{\top},P\mu_G)$ denote the corresponding relabeled graph.
Then $z(G')=z(G)$.
The proof follows from the permutation equivariance of the ART couplings
in Lemma~\ref{lem:art_equivariance}.
\subsection{Computational Complexity}
\label{sec:complexity}

We analyze the computational cost with fixed Sinkhorn iterations.
For the sparse binary adjacency matrix $A$, the relational-energy computation uses the exact identities
\begin{align}
(A\widehat T_k)(i,a)
&=\sum_{j:(i,j)\in E}\widehat T_k(j,a),\\
\sum_{i,j}A_{ij}^2\mu_G(i)\mu_G(j) &=\sum_{(i,j)\in E}\mu_G(i)\mu_G(j).
\label{eq:sparse_energy_identity}
\end{align}
These identities enable exact sparse evaluation without $O(N^2)$ relation storage.

Treating the hidden dimension $D_h$ and the number of GIN layers as
constants, the leading per-graph computational costs are summarized in
Table~\ref{tab:complexity}.
\begin{table}[t]
\centering
\caption{Leading per-graph computational complexity of SCGFM-ART.}
\label{tab:complexity}
\begin{tabular}{lc}
\hline
\textbf{Module} & \textbf{Time Complexity} \\
\hline
Sparse graph encoder
& $O(N+|E|)$ \\
Base encoder
& $O(KM^2)$ \\
Compatibility and Sinkhorn
& $O(SKNM)$ \\
Sparse relational energy
& $O\!\left(K(|E|M+NM^2)\right)$ \\
\hline
\end{tabular}
\end{table}
The resulting per-graph forward complexity is
\begin{equation}
O\left(N+|E|+KM^2+SKNM+K|E|M+KNM^2\right).
\label{eq:total_time_complexity}
\end{equation}
For fixed architectural parameters $K$, $M$, $S$, and $D_h$, the
forward complexity scales linearly with the sparse graph size,
i.e., $O(N+|E|)$. Unlike iterative GW solvers, ART amortizes
graph-to-base matching into a fixed-depth predictor followed by $S$
Sinkhorn iterations, eliminating sample-specific outer optimization.

The per-graph memory complexity is
\begin{equation}
O\left(
N+|E|+KNM+KM^2
\right),
\label{eq:memory_complexity}
\end{equation}
dominated by the sparse graph representation, graph-to-base couplings,
and relational bases.
In comparison, SCGFM incurs ($O\left(KL(N\log N+M\log M)+K^2+NM\cdot\mathrm{iter}\right)$) inference cost due to iterative GW feature projection.
ART removes this sample-specific outer optimization and achieves ($O(N+|E|)$) scaling for fixed architectural parameters.
\begin{table*}[t]
\centering
\caption{Cross-domain 5-shot graph classification results using a prototype classifier.
Accuracy (\%) is reported as mean $\pm$ standard deviation over 50 few-shot episodes.
Best results are shown in bold. }
\label{tab:graph_classification}
\setlength{\tabcolsep}{3pt}
\begin{tabular}{@{}lccccccccc@{}}
\hline
\textbf{Method}
& \textbf{NCI1}
& \textbf{BZR}
& \textbf{COLLAB}
& \textbf{IMDB-B}
& \textbf{PROTEINS}
& \textbf{COLORS-3}
& \textbf{ogbg-molhiv}
& \textbf{Avg.}
& \textbf{Rank$\downarrow$} \\
\hline

GCN~\cite{kipf2017gcn}
& $52.86 \pm 5.41$
& $58.66 \pm 9.96$
& $54.08 \pm 7.27$
& $53.32 \pm 9.61$
& $49.96 \pm 5.97$
& $13.12 \pm 1.86$
& $52.64 \pm 7.22$
& 47.81
& 6.71 \\

GAT~\cite{velickovic2018gat}
& $52.50 \pm 6.30$
& $59.54 \pm 9.96$
& $44.39 \pm 3.86$
& $51.94 \pm 5.19$
& $59.40 \pm 9.01$
& $11.30 \pm 1.35$
& $54.12 \pm 7.33$
& 47.60
& 7.14 \\

GIN~\cite{xu2019gin}
& $50.50 \pm 6.51$
& $49.40 \pm 4.04$
& $50.79 \pm 9.42$
& $\mathbf{60.26 \pm 8.81}$
& $55.18 \pm 6.86$
& $13.13 \pm 2.14$
& $51.94 \pm 5.55$
& 47.31
& 8.00 \\
\hline
GraphCL~\cite{you2020graphcl}
& $\mathbf{59.24 \pm 9.47}$
& $58.36 \pm 8.79$
& $46.76 \pm 5.88$
& $49.46 \pm 5.37$
& $62.10 \pm 9.02$
& $9.89 \pm 1.35$
& $54.98 \pm 8.44$
& 48.68
& 6.43 \\

GraphMAE~\cite{hou2022graphmae}
& $50.34 \pm 6.28$
& $51.28 \pm 5.11$
& $63.71 \pm 6.11$
& $58.10 \pm 10.08$
& $52.80 \pm 7.00$
& $12.93 \pm 1.46$
& $54.08 \pm 7.64$
& 49.03
& 7.00 \\

GraphACL~\cite{xiao2023graphacl}
& $50.80 \pm 5.36$
& $58.18 \pm 8.88$
& $61.31 \pm 5.15$
& $54.46 \pm 6.86$
& $52.18 \pm 5.66$
& $14.08 \pm 1.84$
& $51.54 \pm 6.43$
& 48.94
& 7.14 \\
\hline
GraphGlue~\cite{sun2026graphglue}
& $50.90 \pm 4.94$
& $55.98 \pm 7.18$
& $56.96 \pm 7.02$
& $51.86 \pm 9.71$
& $57.32 \pm 7.33$
& $12.68 \pm 1.82$
& $52.44 \pm 7.45$
& 48.31
& 8.14 \\

GCOPE~\cite{zhao2024gcope}
& $58.00 \pm 7.27$
& $53.16 \pm 7.91$
& $45.40 \pm 6.12$
& $48.40 \pm 4.81$
& $59.84 \pm 10.58$
& $9.61 \pm 1.30$
& $53.92 \pm 8.36$
& 46.90
& 8.71 \\

GiT~\cite{wang2025git}
& $50.32 \pm 5.17$
& $59.88 \pm 11.02$
& $56.55 \pm 6.61$
& $51.92 \pm 7.44$
& $54.80 \pm 7.66$
& $12.89 \pm 1.99$
& $50.74 \pm 5.83$
& 48.16
& 8.57 \\

MDGMIX~\cite{zheng2026mdgmix}
& $51.16 \pm 3.72$
& $53.46 \pm 6.77$
& $56.39 \pm 6.36$
& $52.10 \pm 9.35$
& $52.48 \pm 6.16$
& $10.77 \pm 2.02$
& $51.82 \pm 4.98$
& 46.88
& 9.29 \\

RiemannGFM~\cite{sun2025riemanngfm}
& $51.88 \pm 6.03$
& $52.42 \pm 4.90$
& $50.77 \pm 4.56$
& $52.12 \pm 5.42$
& $60.30 \pm 7.13$
& $11.26 \pm 1.36$
& $53.66 \pm 5.40$
& 47.49
& 8.00 \\

\hline

SCGFM~\cite{he2026scgfm}
& $57.98 \pm 8.04$
& $56.82 \pm 6.97$
& $63.16 \pm 5.67$
& $52.92 \pm 6.46$
& $\mathbf{65.76 \pm 5.82}$
& $23.11 \pm 2.86$
& $54.38 \pm 8.51$
& 53.45
& 3.57 \\

\textbf{SCGFM-ART}
& $56.30 \pm 6.73$
& $\mathbf{62.36 \pm 6.60}$
& $\mathbf{66.33 \pm 3.39}$
& $52.96 \pm 8.43$
& $61.06 \pm 9.67$
& $\mathbf{25.31 \pm 2.72}$
& $\mathbf{55.40 \pm 7.35}$
& $\mathbf{54.25}$
& $\mathbf{2.29}$ \\

\hline
\end{tabular}
\end{table*}

\begin{table*}[t]
\centering
\caption{Cross-domain 5-shot node classification results using a linear classifier.
Accuracy (\%) is reported as mean $\pm$ standard deviation over 50 few-shot episodes.
Best results are shown in bold.}
\label{tab:node_classification}
\setlength{\tabcolsep}{3pt}
\begin{tabular}[\columnwidth]{@{}lccccccccc@{}}
\hline
\textbf{Method}
& \textbf{Cora}
& \textbf{CiteSeer}
& \textbf{PubMed}
& \textbf{Computers}
& \textbf{Photo}
& \textbf{Reddit}
& \textbf{ogbn-arxiv}
& \textbf{Avg.}
& \textbf{Rank$\downarrow$} \\
\hline

GCN~\cite{kipf2017gcn}
& $60.16 \pm 3.67$
& $34.32 \pm 4.27$
& $52.72 \pm 7.52$
& $70.38 \pm 2.89$
& $78.65 \pm 3.73$
& $66.45 \pm 1.73$
& $29.52 \pm 1.10$
& 56.03
& 8.43 \\

GAT~\cite{velickovic2018gat}
& $61.89 \pm 4.37$
& $37.03 \pm 5.23$
& $62.61 \pm 7.50$
& $61.74 \pm 3.13$
& $67.50 \pm 3.87$
& $19.39 \pm 1.12$
& $28.85 \pm 1.62$
& 48.43
& 8.86 \\

GIN~\cite{xu2019gin}
& $59.68 \pm 4.36$
& $33.99 \pm 4.27$
& $48.59 \pm 6.72$
& $59.68 \pm 2.74$
& $19.91 \pm 2.02$
& $32.72 \pm 1.97$
& $14.94 \pm 1.13$
& 38.50
& 11.43 \\

\hline

GraphCL~\cite{you2020graphcl}
& $63.24 \pm 3.72$
& $37.26 \pm 4.92$
& $62.15 \pm 7.61$
& $70.72 \pm 3.00$
& $79.00 \pm 3.66$
& $69.44 \pm 1.99$
& $31.44 \pm 1.55$
& 59.04
& 6.00 \\

GraphMAE~\cite{hou2022graphmae}
& $63.59 \pm 4.25$
& $38.14 \pm 5.74$
& $62.07 \pm 7.97$
& $74.82 \pm 2.53$
& $80.28 \pm 2.85$
& $74.30 \pm 1.69$
& $31.50 \pm 1.52$
& 60.67
& 4.00 \\

GraphACL~\cite{xiao2023graphacl}
& $64.91 \pm 4.10$
& $37.73 \pm 4.52$
& $42.36 \pm 5.26$
& $74.09 \pm 2.93$
& $81.87 \pm 3.41$
& $77.26 \pm 1.43$
& $31.89 \pm 1.60$
& 58.59
& 4.57 \\

\hline

GraphGlue~\cite{sun2026graphglue}
& $57.78 \pm 4.31$
& $34.84 \pm 4.50$
& $52.33 \pm 6.15$
& $67.22 \pm 3.12$
& $72.17 \pm 4.32$
& $53.02 \pm 1.74$
& $21.59 \pm 1.32$
& 51.28
& 10.00 \\

GCOPE~\cite{zhao2024gcope}
& $62.28 \pm 4.61$
& $38.68 \pm 4.76$
& $58.17 \pm 8.83$
& $72.13 \pm 2.40$
& $79.02 \pm 3.46$
& $74.14 \pm 1.59$
& $30.28 \pm 1.47$
& 59.24
& 5.43 \\

GiT~\cite{wang2025git}
& $62.70 \pm 4.69$
& $35.04 \pm 4.78$
& $51.61 \pm 5.90$
& $71.22 \pm 2.90$
& $74.75 \pm 4.06$
& $55.91 \pm 1.71$
& $25.58 \pm 1.43$
& 53.83
& 8.29 \\

MDGMIX~\cite{zheng2026mdgmix}
& $40.06 \pm 4.34$
& $27.19 \pm 4.11$
& $36.56 \pm 5.11$
& $47.38 \pm 4.09$
& $52.04 \pm 3.95$
& $22.90 \pm 1.47$
& $7.68 \pm 1.01$
& 33.40
& 12.71 \\

RiemannGFM~\cite{sun2025riemanngfm}
& $57.37 \pm 4.65$
& $37.30 \pm 4.81$
& $46.29 \pm 6.81$
& $72.00 \pm 3.47$
& $79.23 \pm 3.89$
& $66.99 \pm 1.58$
& $23.62 \pm 1.59$
& 54.69
& 8.14 \\

\hline
SCGFM~\cite{he2026scgfm}
& $67.55 \pm 4.13$
& $40.31 \pm 4.48$
& $\mathbf{66.97 \pm 7.46}$
& $75.76 \pm 3.56$
& $81.69 \pm 3.40$
& $83.76 \pm 1.11$
& $33.65 \pm 1.69$
& 64.24
& 2.00 \\

\textbf{SCGFM-ART}
& $\mathbf{68.14 \pm 3.84}$
& $\mathbf{41.62 \pm 4.18}$
& $66.48 \pm 5.78$
& $\mathbf{76.16 \pm 3.00}$
& $\mathbf{83.22 \pm 2.43}$
& $\mathbf{84.00 \pm 1.03}$
& $\mathbf{35.46 \pm 1.42}$
& $\mathbf{65.01}$
& $\mathbf{1.14}$ \\

\hline
\end{tabular}
\end{table*}

\section{Experiments}
\label{sec:Experiments}
We organize the empirical study into five complementary tiers to evaluate SCGFM-ART from transfer performance to mechanism and scalability.

\subsection{Cross-Domain Few-Shot Transfer}
\textbf{Experimental protocol.}
We evaluate all methods under a unified 5-shot cross-domain transfer protocol
Each class provides five support samples and at most 50 disjoint queries, with results averaged over 50 fixed episodes.
All methods share identical support/query splits, and normalization is fitted only on the support set.

For graph classification, we use NCI1~\cite{wale2008descriptor,shervashidze2011wl}, BZR~\cite{sutherland2003spline}, COLLAB, IMDB-BINARY~\cite{yanardag2015deepgraphkernels}, and PROTEINS~\cite{borgwardt2005protein} for leave-one-dataset-out (LODO) evaluation, while COLORS-3~\cite{knyazev2019attention,morris2020tudataset} and ogbg-molhiv~\cite{wu2018moleculenet,hu2020ogb} are evaluated using the checkpoint jointly pretrained on all five source datasets.
We use a prototype classifier~\cite{snell2017prototypical} as the primary graph-level readout and report Accuracy as the main metric.

For node classification, we analogously perform LODO evaluation on Cora, CiteSeer, and PubMed~\cite{yang2016planetoid}, Computers, and Photo~\cite{shchur2018pitfalls}, with Reddit~\cite{hamilton2017graphsage} and ogbn-arxiv~\cite{hu2020ogb} as additional unseen targets.
Each node is represented by its Personalized Page Rank (PPR) subgraph during pretraining~\cite{bojchevski2020ppr}. Since node-feature dimensions vary across datasets and some baselines require a fixed input width, we apply the same fixed Gaussian random projection (seed 42) to all methods, mapping node attributes to 256 dimensions.
We use a frozen linear classifier as the primary node-level readout and report Accuracy.

\textbf{Baselines.}
We compare SCGFM-ART with three groups of baselines:
(1) conventional GNNs, including
GCN~\cite{kipf2017gcn},
GAT~\cite{velickovic2018gat}, and
GIN~\cite{xu2019gin};
(2) graph self-supervised learning methods, including
GraphCL~\cite{you2020graphcl},
GraphMAE~\cite{hou2022graphmae}, and
GraphACL~\cite{xiao2023graphacl}; and
(3) graph foundation models, including
GraphGlue~\cite{sun2026graphglue},
GCOPE~\cite{zhao2024gcope},
GiT~\cite{wang2025git},
MDGMIX~\cite{zheng2026mdgmix},
RiemannGFM~\cite{sun2025riemanngfm}, and
SCGFM~\cite{he2026scgfm}.
For fair comparison, GCN, GAT, and GIN use three message-passing
layers, while methods with a configurable graph encoder use a
three-layer GCN backbone.
We follow the official implementations and recommended training
settings whenever available.

\textbf{Graph classification results.}
As shown in Table~\ref{tab:graph_classification}, SCGFM-ART achieves the highest macro-average Proto Accuracy of \textbf{54.25\%} and the best average rank of \textbf{2.29}.
Compared with SCGFM, it improves the average accuracy by 0.80 percentage points and reduces the average rank from 3.57 to 2.29, demonstrating more consistent cross-domain transfer across the evaluated graph datasets.

\textbf{Node classification results.}
As shown in Table~\ref{tab:node_classification}, SCGFM-ART achieves
the highest macro-average Linear Accuracy of \textbf{65.01\%} and
the best average rank of \textbf{1.14}.
It outperforms SCGFM by 0.77 percentage points on average and improves
the average rank from 2.00 to 1.14, achieving the best performance on
six of the seven target datasets.

Overall, the Tier-A results answer: the relational atlas learned by SCGFM-ART yields a frozen representation that remains competitive across heterogeneous and previously unseen graph domains.
The improvement over SCGFM is particularly reflected in average rank, indicating broader cross-domain competitiveness rather than an improvement confined to a small number of targets.

\subsection{Core Component Analysis}
\label{sec:component_analysis}

We examine which components of SCGFM-ART contribute to the
learned representation and how the parameterization of the relational atlas affects downstream performance.
We select BZR and PROTEINS as representative graph-level datasets and
Cora and Computers as representative node-level datasets.
All component variants follow the same unsupervised pretraining
objective and downstream protocol used by the full model.
Each pretrained model is evaluated over the same 50 deterministic
5-shot episodes and across 5 different random seeds.

\begin{figure}[!t]
\centering
\includegraphics[width=0.90\columnwidth]{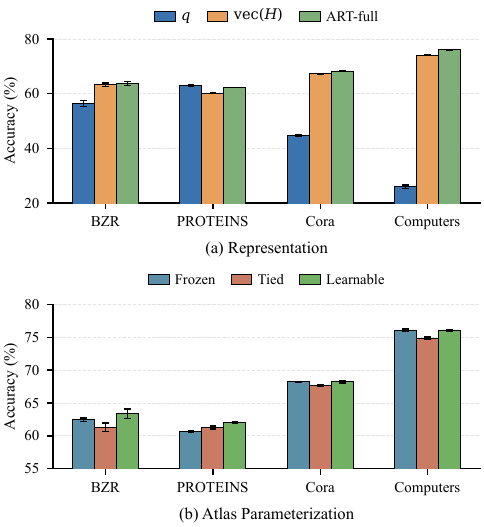}
\caption{
Component analysis of SCGFM-ART on representative graph- and node-level datasets.
Bars denote mean accuracy over five matched pretraining seeds, and error
bars denote standard deviations.
}
\label{fig:rq2_component_analysis}
\end{figure}

\textbf{Representation decomposition.}
SCGFM-ART combines two complementary quantities: the relational atlas
coordinate $q(G)$ and the transport-conditioned feature representation
$H(G)$.
To isolate their contributions without introducing additional
optimization variation, we extract three representations from the same
frozen full-model checkpoint for each seed:
$q(G)$, $\operatorname{vec}(H(G))$, $[q(G)\Vert\operatorname{vec}(H(G))]$.

As shown in Fig.~\ref{fig:rq2_component_analysis}, $H$ carries most transferable information on feature-rich targets.
Compared with the coordinate-only representation, full model improves accuracy by 7.01, 23.50, and 50.12 points, respectively, with particularly large gains on node classification.
In contrast, $q$ contributes more at the graph level: adding it to $\operatorname{vec}(H)$ raises the graph-level average from 61.78\% to 62.32\%, including a 1.02-point gain on PROTEINS.
Notably, PROTEINS favors structural information, where $q$ alone reaches 63.04\%, 1.76 points above ART-full. These results suggest that $q$ captures compact global structure, while $H$ primarily encodes role-conditioned attribute information.

\textbf{Atlas parameterization}.
For the relational atlas, independent learnable bases achieve the best graph-level average of 62.32\%, compared with 61.54\% for frozen random bases and 61.68\% for a tied atlas.
Sharing a single relational kernel across all $K$ slots reduces accuracy by 2.05, 0.54, and 1.20 points on BZR, Cora, and Computers, respectively, supporting the benefit of multiple relational references. On the two node-level targets, however, frozen and learned independent atlases perform almost identically (72.19\% vs. 72.16\%), indicating that graph-to-role transport already provides strong structural routing in feature-rich settings.

\textbf{Atlas capacity.}
We further examine the sensitivity to the number of relational bases
$K$ and the number of roles per base $M$.
Fig.~\ref{fig:km_capacity} varies one dimension while fixing the
other at its default value, $K=16$ and $M=32$.
Graph-level performance is evaluated on BZR and PROTEINS, and
node-level performance on Cora and Computers.

\begin{figure}[t]
\centering
\includegraphics[width=0.8\columnwidth]{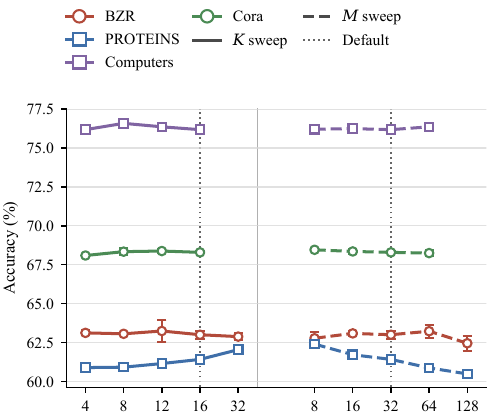}
\caption{
Sensitivity of SCGFM-ART to the atlas capacity on representative graph- and node-level datasets. Solid and dashed curves show the effects of varying the number of bases ($K$) and the base size ($M$), respectively.
Dotted vertical lines mark the default configuration, $K=16$ and $M=32$.
}
\label{fig:km_capacity}
\end{figure}

Fig.~\ref{fig:km_capacity} shows that graph-level transfer is more
sensitive to atlas capacity than node-level transfer.
Increasing $K$ improves PROTEINS while leaving BZR largely unchanged,
whereas both Cora and Computers remain stable over the evaluated range.
This pattern is consistent with $K$ primarily controlling the breadth
of the structural reference system: additional bases can increase
coverage of heterogeneous graph structures, while providing limited
additional benefit when downstream discrimination is already dominated
by transported node attributes.

The effect of $M$ is qualitatively different.
Increasing the number of roles does not yield monotonic improvements:
BZR reaches its highest accuracy at an intermediate resolution, while
PROTEINS gradually decreases as $M$ becomes large.
A larger role space increases the resolution and degrees of freedom of
each graph--base correspondence, but the additional capacity need not
translate into more informative transport when the available structural
signal is limited.
The node-level results are again comparatively insensitive to $M$,
indicating that SCGFM-ART does not rely on a narrowly tuned role
cardinality for these targets.

Overall, $K$ and $M$ are not interchangeable capacity parameters.
$K$ mainly controls the breadth of the relational atlas, whereas $M$
controls the granularity of individual graph--base correspondences.
The default setting $(K,M)=(16,32)$ lies in a stable operating region
across both graph- and node-level tasks, while avoiding the additional
computation associated with unnecessarily large atlases or role spaces.

\subsection{Amortized Relational Alignment}
\label{sec:amortized_alignment}

\textbf{Experimental protocol.}
We evaluate whether ART can replace iterative graph--base alignment while preserving the utility of the resulting representations.
Experiments are conducted on BZR, COLLAB, Cora, and Computers using the corresponding frozen LODO checkpoints.
For each target dataset, we stratify 128 target objects and evaluate all $K=16$ graph--base pairs.
We compare three inference schemes: direct ART prediction, ART followed by 10 iterative refinement steps, and an iterative relational solver used as the numerical reference.
The reference solver runs for 200 outer iterations with at most 500 Sinkhorn iterations per outer step; three deterministic initializations are evaluated for each graph--base pair, and the solution with the lowest relational energy is retained.
Alignment consistency is measured by the Spearman correlation between the $K$ graph--base energy responses and by agreement on the minimum-energy base.
We further reconstruct the downstream representation from the
corresponding couplings and report classification accuracy under the same 5-shot protocol, together with end-to-end latency per graph--base pair.

\begin{table}[t]
\centering
\caption{
Comparison of amortized and iterative relational alignment on BZR, COLLAB, Cora, and Computers.
}
\label{tab:amortized_alignment}
\setlength{\tabcolsep}{2.8pt}
\resizebox{\columnwidth}{!}{
\begin{tabular}{@{}lccccc@{}}
\hline
\textbf{Method}
& $\boldsymbol{\rho}$
& \textbf{Top-1 (\%)}
& \textbf{Acc. (\%)}
& \textbf{ms/pair}
& \textbf{Speedup} \\
\hline
ART
& 0.703
& 41.0
& $\mathbf{68.30}$
& $\mathbf{3.04}$
& $\mathbf{24.3\times}$ \\

ART + 10-step refinement
& $\mathbf{0.997}$
& $\mathbf{99.4}$
& 67.09
& $4.49$
& $16.4\times$ \\

Iterative reference
& 1.000
& 100.0
& 67.09
& 73.79
& $1.0\times$ \\
\hline
\end{tabular}
}
\end{table}

\textbf{Results.}
Table~\ref{tab:amortized_alignment} shows that direct ART reduces the average alignment latency from 73.79 to 3.04 ms per graph--base pair, yielding a $24.3\times$ speedup over the iterative reference.
Despite moderate agreement with the reference
($\rho=0.703$ and 41.0\% Top-1 agreement), ART achieves the highest downstream accuracy of 68.30\%, exceeding the iterative reference by 1.21 percentage points on average.

A small amount of iterative refinement nearly recovers the numerical reference: 10 refinement steps increase $\rho$ from 0.703 to 0.997 and Top-1 agreement from 41.0\% to 99.4\%, while retaining a $16.4\times$ speedup.
Its downstream accuracy, however, decreases to 67.09\%, matching that of the iterative reference.
The same pattern is visible across individual targets.
Direct ART improves over the reference by 2.76 and 2.02 percentage points on BZR and COLLAB, respectively, while the differences on Cora and Computers remain within 0.2 points.

These results show that reproducing the numerical reference more closely does not translate into higher downstream accuracy in this setting.
Direct ART provides the strongest transfer performance at the lowest inference cost, whereas short iterative refinement offers an optional route to near-reference alignment when higher numerical agreement is required.

\subsection{Structure-Conditioned Representation and Atlas Interpretation}
\label{sec:structure_conditioned}

\textbf{Experimental protocol.}
We examine the structural information retained by the frozen ART
representation on unseen target graphs and how this information is
organized across the learned relational atlas.
Experiments use the LODO checkpoints for BZR, COLLAB, Cora, and
Computers, with 128 stratified target objects sampled from each dataset.
The analysis consists of two complementary diagnostics.
First, we apply controlled topology perturbations while preserving node
degrees and attributes, and measure the resulting changes in atlas
coordinates, graph-to-base couplings, and transport-conditioned
features.
Second, we examine how the $K$ relational bases are utilized across
target objects and whether their response profiles remain differentiated.
All model parameters are frozen throughout the analysis.

\textbf{Controlled topology perturbation.}
For each target object, we generate perturbed graphs using undirected
double-edge swaps with requested rewiring ratios $p\in\{0,0.1,0.2,0.4,0.6\}$.
The perturbation preserves the degree of every node and leaves node attributes unchanged, while progressively modifying graph connectivity.
We report the realized rewiring ratio and quantify representation changes at three stages of SCGFM-ART:
\begin{align}
\Delta_q
&=
1-\cos\!\left(q(G),q(G^{(p)})\right),\\
\Delta_T
&=
\frac{1}{K}\sum_{k=1}^{K}
\frac{
\|\widehat T_k-\widehat T_k^{(p)}\|_F
}{
\|\widehat T_k\|_F+\epsilon
},\\
\Delta_H
&=
1-\cos\!\left(
\operatorname{vec}(H(G)),
\operatorname{vec}(H(G^{(p)}))
\right), 
\end{align}
where $\cos(\cdot,\cdot)$ denotes cosine similarity, and $\epsilon>0$ is a small constant for numerical stability.
The three rewiring repeats are first averaged within each object, after
which 95\% confidence intervals are computed over the 128 target objects.
A structure-independent attribute-pooling representation is evaluated under the same perturbations as a control.

\begin{figure}[t]
\centering
\includegraphics[width=\columnwidth]{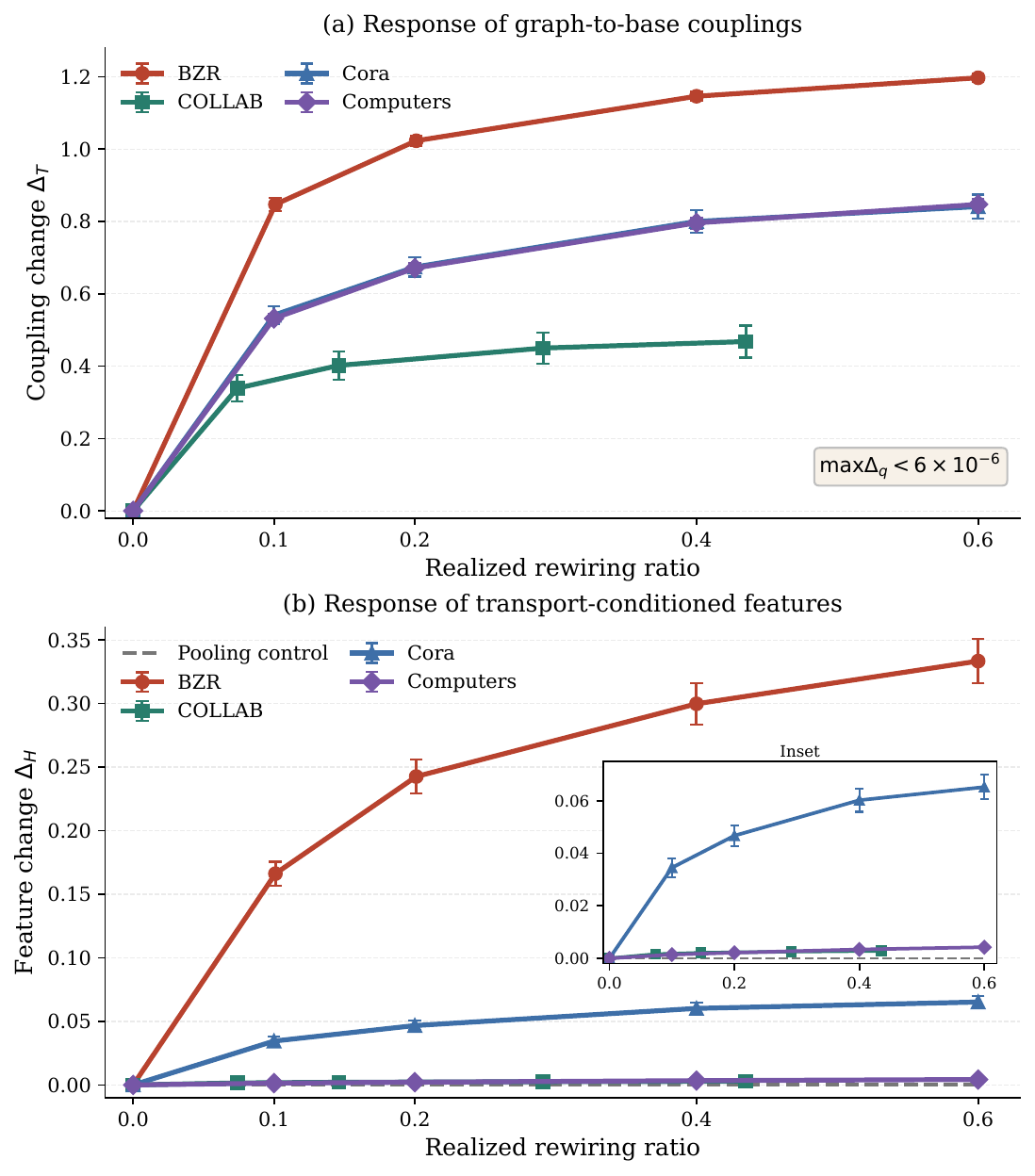}
\caption{
Response of the frozen ART representation to degree-preserving topology
perturbations.
The horizontal axis reports the realized rewiring ratio.
(a) Relative variation $\Delta_T$ of graph-to-base couplings.
(b) Variation $\Delta_H$ of the transport-conditioned feature
representation.
Error bars denote 95\% confidence intervals over 128 target objects
after averaging three rewiring repeats within each object.
The structure-independent pooling control remains unchanged, while
$\Delta_q$ stays below $6\times10^{-6}$ across all perturbation levels.
}
\label{fig:topology_perturbation}
\end{figure}

\textbf{Perturbation results.}
Fig.~\ref{fig:topology_perturbation} (a) shows that coupling variation increases consistently as progressively more edges are rewired.
At the largest realized perturbation, $\Delta_T$ reaches
$1.197\pm0.011$ on BZR,
$0.468\pm0.044$ on COLLAB,
$0.841\pm0.033$ on Cora, and
$0.847\pm0.015$ on Computers.
In contrast, $\Delta_q$ remains below $6\times10^{-6}$ throughout the experiment.
The graph-to-base energy responses are therefore considerably more stable under degree-preserving rewiring than the underlying node-to-role correspondences.

The response of $H$ exhibits stronger dataset dependence, even when
the corresponding changes in ART couplings are comparable.
This behavior is consistent with Eq. \eqref{eq:feature_recoding} 
where topology perturbations alter the transport assignments encoded by $T_{\mathrm{mix}}$, while the resulting feature-space variation also depends on the distribution of node attributes over the reassigned nodes.
Consequently, comparable changes in coupling space can produce markedly
different responses in $H$ across target domains.
The invariant attribute-pooling control further localizes this effect
to the structure-conditioned redistribution of fixed node attributes.

These observations separate two levels of structural information in SCGFM-ART.
The coordinate $q$ provides a stable global response to the relational atlas, whereas the couplings retain substantially finer variation in
node-to-role correspondence.
The latter variation is subsequently reflected in $H$ according to the interaction between the learned transport and target attributes.

\textbf{Relational-atlas interpretation.}
We examine whether the differentiated correspondence behavior
observed above is supported by a diverse utilization of the learned
atlas.
For each target dataset, we compute the average responsibility of base
$k$ as
\begin{equation}
\bar w_k
=
\frac{1}{|\mathcal G|}
\sum_{G\in\mathcal G} w_k(G),
\end{equation}
and summarize the resulting distribution by the effective number of
bases
\begin{equation}
N_{\mathrm{eff}}
=
\exp\left(
-\sum_{k=1}^{K}
\bar w_k\log \bar w_k
\right).
\end{equation}
For each base, we additionally retain the eight target objects with the
largest responsibilities.
We measure the mean pairwise Jaccard overlap between these top-response
sets and the number of distinct objects contained in their union.
Finally, we compute the maximum total-variation (TV) distance between
the responsibility-weighted class composition associated with an
individual base and the overall class distribution of the dataset.

\begin{table}[t]
\centering
\caption{
Utilization and response differentiation of the learned relational atlas.
$N_{\mathrm{eff}}$ denotes the effective number of bases.
Jaccard is the mean pairwise overlap between the top-8 response sets of
different bases, and Unique is the size of their union.
Max TV denotes the largest deviation between a base-specific
responsibility-weighted class composition and the dataset-level class
distribution.
}
\label{tab:atlas_interpretation}
\setlength{\tabcolsep}{3.5pt}
\resizebox{\columnwidth}{!}{
\begin{tabular}{@{}lcccc@{}}
\hline
\textbf{Dataset}
& $\boldsymbol{N_{\mathrm{eff}}}$
& \textbf{Top-8 Jaccard}
& \textbf{Unique}
& \textbf{Max TV} \\
\hline
BZR       & 16.000 & 0.053 & 73 & 0.0004 \\
COLLAB    & 15.998 & 0.046 & 74 & 0.0140 \\
Cora      & 16.000 & 0.075 & 60 & 0.0023 \\
Computers & 15.996 & 0.059 & 71 & 0.0051 \\
\hline
\end{tabular}
}
\end{table}

\textbf{Atlas results.}
Table~\ref{tab:atlas_interpretation} shows that the learned atlas remains broadly utilized across all four target datasets.
The effective number of bases are all close to the maximum value $K=16$, indicating that responsibility mass remains distributed across the atlas rather than concentrating on a small subset of bases.

The object-level response profiles are considerably more differentiated.
The mean pairwise overlap between the top-8 response sets is only
0.046--0.075, while their unions cover 60--74 of the 128 sampled
objects.
Thus, similar aggregate utilization across bases coexists with distinct
high-response object sets.
The learned atlas consequently maintains broad global participation
while preserving base-dependent responses at the level of individual
target structures.

The class-composition diagnostic provides an additional characterization
of this differentiation.
The maximum TV distance remains below 0.015 on every dataset, showing
that the different response profiles are accompanied by only minor
changes in class composition.
The learned bases therefore organize relational variation that extends
across target classes rather than simply partitioning objects according
to downstream labels.

Taken together, the two diagnostics characterize complementary aspects
of the ART representation.
Degree-preserving perturbations reveal fine-grained structural
sensitivity in the graph-to-role couplings beyond the comparatively
stable atlas coordinates, while the atlas-level statistics show that
these correspondences are supported by broadly utilized and
differentiated relational references.
This provides empirical support for the two-part SCGFM-ART construction:
$q$ summarizes global graph-to-atlas responses, whereas the learned
couplings retain finer structural correspondence that conditions the
transported feature representation $H$.

\subsection{End-to-End Efficiency and Scalability}
\label{sec:efficiency_scalability}

\textbf{Experimental protocol.}
We examine how amortized relational transport changes the computational profile and scalability of SCGFM.
The controlled synthetic experiment therefore focuses on SCGFM and SCGFM-ART, isolating the effect of replacing the per-instance geometric alignment in SCGFM with amortized graph--base prediction.
We evaluate two quantities that directly characterize this change: frozen-inference latency and peak training memory.
The former measures the transfer-time cost of relational alignment, while the latter measures the additional memory required to maintain and optimize graph--base couplings during atlas learning.

For broader computational context, we additionally benchmark GraphGlue, GCOPE, GiT, MDGMIX, RiemannGFM, SCGFM, and SCGFM-ART on COLLAB graphs and Reddit PPR subgraphs.
All methods process the same saved graph instances within each benchmark.
Each configuration is executed in a separate process with CUDA synchronization and repeated five times.
The synthetic benchmark uses 10 warm-up and 30 timed iterations with batch size 32, while the real-world benchmark uses 10 warm-up and 20 timed iterations with batch size 64.
Runtime statistics are reported as medians over the five repeats, and
memory denotes peak allocated GPU memory.

\textbf{Computational effect of amortized relational transport.}
We generate synthetic sparse graphs with fixed average degree $\bar d=8$ and vary the number of node $N$.
The same graph batch is used by SCGFM and SCGFM-ART at each scale.

\begin{figure}[t]
\centering
\includegraphics[width=0.85\columnwidth]{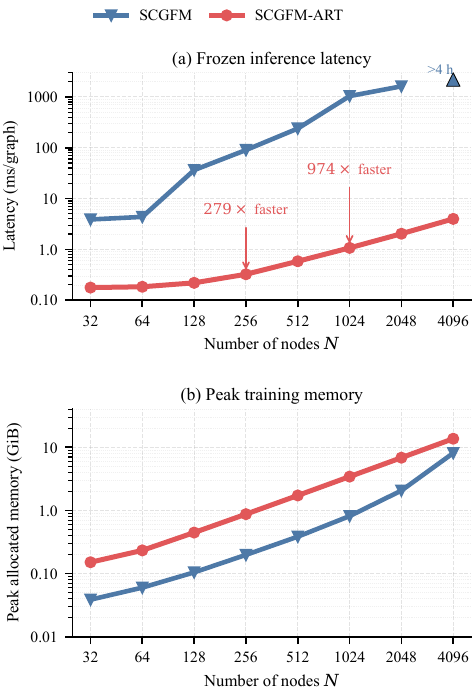}
\caption{
Controlled scalability of SCGFM and SCGFM-ART on synthetic sparse graphs with average degree $\bar d=8$.
(a) Frozen-inference latency as graph size increases.
(b) Peak allocated GPU memory during training.
}
\label{fig:rq5_scalability}
\end{figure}

\textbf{Synthetic scaling results.}
Fig.~\ref{fig:rq5_scalability} (a) shows that the computational effect of ART.
At $N=256$, SCGFM-ART reduces the inference latency of SCGFM by $279\times$, and the reduction reaches $974\times$ at $N=1024$.
At $N=2048$, the corresponding latencies are 2.031 and 1621.75 ms per graph.
The separation continues to increase at the largest scale: the SCGFM run at $N=4096$ exceeds four hours totally, whereas SCGFM-ART completes inference in 3.979 ms per graph.
This trend reflects the replacement of repeated target-specific alignment with an amortized prediction path whose execution depth does not increase with graph size.

Fig.~\ref{fig:rq5_scalability} (b) characterizes the corresponding space cost during pretraining.
SCGFM-ART requires additional memory to retain the $K$ graph--base couplings and their intermediate computations, but its memory growth remains regular over the full range.
From $N=256$ to $4096$, graph size increases by $16\times$, while peak training memory increases from 0.874 to 13.69 GiB, approximately $15.7\times$.
The empirical log--log scaling exponent is 0.99, closely matching the linear dependence on sparse graph size predicted for fixed $K$ and $M$.
Thus, the additional coupling memory introduced by ART changes the
constant computational cost of pretraining without introducing a higher-order dependence on $N$ over the evaluated regime.

Taken together, the two panels expose the computational trade-off introduced by amortization.
SCGFM-ART allocates additional memory to learn graph--base correspondences during pretraining, while eliminating the rapidly growing per-instance alignment cost during frozen transfer.
The resulting computation is shifted from repeated target-time optimization to the reusable pretrained ART predictor.

\textbf{Real-world efficiency.}
We further evaluate the same computational trade-off on COLLAB and Reddit PPR subgraphs.
The real-world benchmark reports training throughput, frozen-inference latency, and peak training memory, and additionally includes representative GFMs to provide a broader runtime context.

\begin{table}[t]
\centering
\caption{
End-to-end efficiency on COLLAB and Reddit PPR subgraphs.
}
\label{tab:real_efficiency}
\resizebox{\columnwidth}{!}{
    \begin{tabular}{@{}llrrr@{}}
    \hline
    \textbf{Dataset}
    & \textbf{Method}
    & \textbf{Train $\uparrow$}
    & \textbf{Infer. $\downarrow$}
    & \textbf{Mem. $\downarrow$} \\
    &
    & \textbf{(G/s)}
    & \textbf{(ms)}
    & \textbf{(GiB)} \\
    \hline
    
    COLLAB & GraphGlue
    & 4917.13 & 0.087 & 0.239 \\
    COLLAB & GCOPE
    & 2697.96 & 0.152 & 0.102 \\
    COLLAB & GiT
    & 271.45 & 0.026 & 0.095 \\
    COLLAB & MDGMIX
    & 1409.52 & $\mathbf{0.025}$ & $\mathbf{0.029}$ \\
    COLLAB & RiemannGFM
    & 199.91 & 4.298 & 0.773 \\
    COLLAB & SCGFM
    & $\mathbf{10191.93} $& 11.567 & 0.296 \\
    COLLAB & \textbf{SCGFM-ART}
    & 1794.78 & 0.262 & 1.396 \\
    \hline
    
    Reddit & GraphGlue
    & 983.12 & 0.328 & 5.127 \\
    Reddit & GCOPE
    & 873.64 & 0.371 & 2.208 \\
    Reddit & GiT
    & 56.89 & $\mathbf{0.234}$ & 2.065 \\
    Reddit & MDGMIX
    & 1282.96 & 0.235 & $\mathbf{0.362}$ \\
    Reddit & RiemannGFM
    & 143.84 & 4.623 & 26.862 \\
    Reddit & SCGFM
    & $\mathbf{7391.40}$ & 139.843 & 0.770 \\
    Reddit & \textbf{SCGFM-ART}
    & 302.53 & 1.643 & 4.028 \\
    \hline
    \end{tabular}
}
\end{table}

Table~\ref{tab:real_efficiency} confirms that the transfer-time gains of amortized relational transport persist on real graph workloads.
Compared with SCGFM, SCGFM-ART reduces frozen-inference latency by $44.2\times$ on COLLAB and $85.1\times$ on Reddit.
The larger gain on Reddit is consistent with the widening separation observed with increasing graph size in Fig.~\ref{fig:rq5_scalability}, showing that the benefit of amortization becomes more pronounced on larger target graphs.

SCGFM-ART shifts relational alignment from repeated target-time optimization into the learned atlas and ART predictor.
This shift increases the computation and memory required during pretraining, while substantially reducing the recurring cost of frozen transfer.
Together with the controlled scaling results, the real-world benchmarks show that ART converts the dominant per-instance alignment cost of SCGFM into a reusable prediction process, yielding a substantially more scalable computational profile for structure-centric graph transfer.

\section{Conclusion}
We presented SCGFM-ART, a structure-centric foundation framework that unifies heterogeneous graphs by mapping them onto a shared relational atlas via Amortized Relational Transport (ART).
This dual formulation captures global atlas response coordinates alongside fine-grained node-to-role structural correspondences, providing a canonical reference frame for cross-domain transfer.
We established theoretical guarantees for coordinate fidelity and coverage bounds, proving that our amortized objective reliably proxies ideal relational alignment.Across 14 cross-domain benchmarks, SCGFM-ART achieves state-of-the-art few-shot transfer performance at both node and graph levels while preserving fine-grained topological nuances.
By replacing iterative target-time alignment with amortized prediction, SCGFM-ART scales linearly with sparse graph size and delivers 44.2×–85.1× faster frozen inference on real-world benchmarks, establishing relational atlases and amortized transport as a scalable paradigm for graph foundation models.

\section*{Acknowledgments}
This work was supported by the National Natural Science
Foundation of China (No. U24A20323).

\clearpage
\section*{Supplementary Material for SCGFM-ART}
\appendices

\setcounter{theorem}{0}
\setcounter{lemma}{0}
\setcounter{proposition}{0}
\setcounter{corollary}{0}
\setcounter{definition}{0}
\setcounter{assumption}{0}

\section{Dataset Details}
\label{app:dataset_details}

\begin{table*}[t]
\centering
\caption{Statistics and node-feature configurations of the
graph-classification datasets. $|V|$ denotes the number of graphs.
Node and edge counts are averages per graph, with each undirected edge counted
once.}
\label{tab:app_graph_datasets}
\begin{tabular}[width=\linewidth]{@{}lrrrrl@{}}
\hline
\textbf{Dataset}
& $|V|$
& \textbf{Classes}
& \textbf{Avg. nodes}
& \textbf{Avg. edges}
& \textbf{Role} \\
\hline
NCI1
& 4,110
& 2
& 29.87
& 32.30
& LODO source/target \\

BZR
& 405
& 2
& 35.75
& 38.36
& LODO source/target \\

COLLAB
& 5,000
& 3
& 74.49
& 2,457.22
& LODO source/target \\

IMDB-BINARY
& 1,000
& 2
& 19.77
& 96.53
& LODO source/target \\

PROTEINS
& 1,113
& 2
& 39.06
& 72.82
& LODO source/target \\

COLORS-3
& 10,500
& 11
& 61.31
& 91.03
& Unseen target \\

\texttt{ogbg-molhiv}
& 41,127
& 2
& 25.51
& 27.47
& Unseen target \\
\hline
\end{tabular}
\end{table*}
We evaluate SCGFM-ART on 14 datasets spanning molecular, biological,
social, citation, and product co-purchase graphs. The graph-classification
benchmarks contain independent graphs with graph-level labels, whereas each
node-classification benchmark is a single large graph with labels attached to
its nodes. Statistics in Tables~\ref{tab:app_graph_datasets}
and~\ref{tab:app_node_datasets} are computed before the episodic sampling used
in our experiments. For the TU datasets, an edge is counted once; for the
node-classification datasets, the edge counts follow the PyTorch Geometric
storage convention, under which an undirected edge is represented in both
directions. The edge count of \texttt{ogbn-arxiv} is instead the number of
directed citation edges.

\subsection{Graph-Classification Datasets}

The TU datasets are distributed through the TUDataset collection
\cite{morris2020tudataset}. NCI1, BZR, COLLAB, IMDB-BINARY, and PROTEINS serve
as the five source domains in leave-one-dataset-out (LODO) evaluation: when one
of these datasets is the target, the model is pretrained on the other four.
For the two additional unseen targets, COLORS-3 and \texttt{ogbg-molhiv}, we
use the checkpoint jointly pretrained on all five source datasets. Each target
class contributes five support graphs and at most 50 disjoint query graphs in
each episode, as specified in the main paper.

\begin{itemize}
    \item \textbf{NCI1.} NCI1 is a molecular graph benchmark derived from
    chemical compounds screened for anti-cancer activity
    \cite{wale2008descriptor,shervashidze2011wl}. Each node represents an atom,
    each edge represents a chemical bond, and the categorical node label
    specifies the atom type. The 4,110 compounds form a binary graph
    classification task according to their assay activity. NCI1 is relatively
    sparse and is used to test transfer to molecular topology without
    continuous node attributes.

    \item \textbf{BZR.} BZR contains 405 molecular compounds collected for a
    quantitative structure--activity relationship study
    \cite{sutherland2003spline}. Graph labels distinguish active and inactive
    ligands of the benzodiazepine receptor. Atoms and chemical bonds form the
    nodes and edges, respectively. In addition to categorical atom labels, the
    TU release provides three-dimensional node coordinates; consequently, BZR
    supplies both discrete identity and continuous geometric attributes.

    \item \textbf{COLLAB.} COLLAB is a scientific collaboration benchmark
    introduced with Deep Graph Kernels \cite{yanardag2015deepgraphkernels}.
    Each graph is an ego-network of a researcher, nodes are researchers, and an
    edge indicates coauthorship. The target is the research field of the ego
    researcher---high-energy physics, condensed-matter physics, or
    astrophysics---yielding three classes. COLLAB has no intrinsic node labels
    or attributes and is much denser than the other graph-level datasets.

    \item \textbf{IMDB-BINARY.} IMDB-BINARY is also an ego-network benchmark
    \cite{yanardag2015deepgraphkernels}. Nodes denote actors and an edge joins
    two actors who appeared in the same film. Each of the 1,000 graphs is
    labeled by one of two movie genres (Action or Romance). The dataset does
    not provide intrinsic node features, making its prediction primarily
    dependent on collaboration structure.

    \item \textbf{PROTEINS.} PROTEINS contains 1,113 protein graphs and asks
    whether a protein is an enzyme \cite{borgwardt2005protein}. Nodes represent
    secondary-structure elements, and edges encode neighborhood relations in
    the amino-acid sequence or in three-dimensional space. The TU release
    provides categorical secondary-structure labels together with one
    continuous node attribute. This dataset therefore combines biological
    topology with node-side information.

    \item \textbf{COLORS-3.} COLORS-3 is a synthetic structural benchmark
    introduced by Knyazev \emph{et al.} \cite{knyazev2019attention}. It contains
    10,500 randomly generated graphs with color-derived node attributes. The
    graph label is determined by the number of green nodes, producing 11
    classes. Because the relevant evidence must be aggregated over the entire
    graph, COLORS-3 tests whether a representation learned from the five
    real-world source domains can transfer to a controlled counting task.

    \item \textbf{ogbg-molhiv.} \texttt{ogbg-molhiv} is the OGB version of the
    HIV molecular-property dataset from MoleculeNet
    \cite{wu2018moleculenet,hu2020ogb}. It contains 41,127 molecules processed
    with RDKit. Nodes are atoms and edges are chemical bonds. The dataset provides nine
categorical atom-feature fields (e.g., atomic number, chirality, formal
charge, and ring membership), together with categorical bond attributes.
In our experiments, however, the encoders use the atom features and
unweighted graph connectivity; the bond attributes are not included in the
model input.
    The binary label indicates whether a molecule inhibits HIV replication.
    OGB defines a scaffold split and ROC-AUC as its standard benchmark
    protocol; in our cross-domain study, however, it is treated as an unseen
    target under the same episodic 5-shot protocol as the other datasets.
\end{itemize}

\subsection{Node-Classification Datasets}

Cora, CiteSeer, PubMed, Computers, and Photo form the five node-level domains
used for LODO pretraining and evaluation. Reddit and \texttt{ogbn-arxiv} are
additional unseen targets evaluated using the checkpoint pretrained jointly on
the five source domains. To expose a common graph-level interface during
pretraining, we represent each target node by its Personalized PageRank (PPR)
subgraph \cite{bojchevski2020ppr}. As raw feature widths differ markedly across
domains we construct a deterministic Gaussian projection matrix for each
input width using the same random seed (42), and map all node features to 256
dimensions. The original feature dimensions reported
below are those before this projection.

\begin{table*}[t]
\centering
\caption{Statistics of the node-classification datasets. The counts follow the
processed PyTorch Geometric/OGB versions used in our experiments. For the six
undirected PyTorch Geometric graphs, the table reports stored directed edge
entries; \texttt{ogbn-arxiv} retains directed citation edges.}
\label{tab:app_node_datasets}
\begin{tabular}{@{}lrrrrl@{}}
\hline
\textbf{Dataset} & \textbf{Nodes} & \textbf{Edges} & \textbf{Features}
& \textbf{Classes} & \textbf{Role} \\
\hline
Cora             & 2,708   & 10,556      & 1,433 & 7  & LODO source/target \\
CiteSeer         & 3,327   & 9,104       & 3,703 & 6  & LODO source/target \\
PubMed           & 19,717  & 88,648      & 500   & 3  & LODO source/target \\
Computers        & 13,752  & 491,722     & 767   & 10 & LODO source/target \\
Photo            & 7,650   & 238,162     & 745   & 8  & LODO source/target \\
Reddit           & 232,965 & 114,615,892 & 602   & 41 & Unseen target \\
\texttt{ogbn-arxiv} & 169,343 & 1,166,243 & 128 & 40 & Unseen target \\
\hline
\end{tabular}
\end{table*}

\begin{itemize}
    \item \textbf{Cora, CiteSeer, and PubMed.} These Planetoid citation
    networks \cite{yang2016planetoid} represent scientific documents as nodes
    and citation relations as edges. Sparse bag-of-words vectors provide node
    features, and the node labels denote document topics. The processed graphs
    contain 2,708, 3,327, and 19,717 nodes and have 7, 6, and 3 classes,
    respectively. They provide three related but differently sized citation
    domains for measuring cross-domain transfer.

    \item \textbf{Computers and Photo.} The Amazon co-purchase networks were
    introduced for controlled GNN evaluation
    \cite{shchur2018pitfalls}. Nodes are products, and two products are linked
    when they are frequently purchased together. Bag-of-words representations
    of product reviews form the node features, while labels correspond to
    product categories. Computers contains 13,752 products in 10 classes;
    Photo contains 7,650 products in 8 classes.

    \item \textbf{Reddit.} Reddit is the large inductive node-classification
    benchmark released with GraphSAGE \cite{hamilton2017graphsage}. A node is a
    Reddit post, and two posts are connected when the same user comments on
    both. Each post has a 602-dimensional feature vector, and its label is the
    community in which it was posted. The processed graph contains 232,965
    nodes, 114,615,892 stored edge entries, and 41 classes. Its scale and dense
    local neighborhoods make it a challenging unseen transfer domain.

    \item \textbf{ogbn-arxiv.} \texttt{ogbn-arxiv} is a directed citation
    network of 169,343 Computer Science papers indexed by the Microsoft
    Academic Graph \cite{hu2020ogb}. Each node has a 128-dimensional feature
    obtained by averaging skip-gram embeddings of words in its title and
    abstract; each directed edge represents a citation. The prediction target
    is one of 40 primary arXiv subject areas. OGB conventionally uses a temporal
    split (training through 2017, validation in 2018, and testing from 2019
    onward), whereas our experiment treats the graph as an unseen domain and
    constructs fixed 5-shot support/query episodes from disjoint labeled nodes.
\end{itemize}

\section{Experimental Details}
\label{app:experimental_details}

This appendix provides additional implementation and evaluation details for
the cross-domain few-shot transfer experiment reported in the main paper. Unless otherwise specified, the same protocol
is used for all graph- and node-classification targets.

\begin{table*}[t]
\centering
\caption{Default settings for the cross-domain few-shot transfer experiment.}
\label{tab:app_default_transfer_settings}
\setlength{\tabcolsep}{8pt}
\begin{tabular}{@{}lll@{}}
\hline
\textbf{Category} & \textbf{Parameter} & \textbf{Default value} \\
\hline
Few-shot evaluation & Support samples per class & 5 \\
Few-shot evaluation & Maximum query samples per class & 50 \\
Few-shot evaluation & Number of episodes per target & 50 \\
Few-shot evaluation & Graph-level downstream head & Prototype classifier~\cite{snell2017prototypical} \\
Few-shot evaluation & Node-level downstream head & Linear classifier \\
Relational atlas & Number of bases $K$ & 16 \\
Relational atlas & Number of roles per base $M$ & 32 \\
Graph preprocessing & Unified node-feature width & 56 \\
Node preprocessing & Projected node-feature width & 256 \\
Node preprocessing & Gaussian projection seed & 42 \\
Node preprocessing & Maximum PPR ego-graph size & 400 nodes \\
Node preprocessing & PPR restart probability $\alpha$ & 0.15 \\
Node preprocessing & Local-push residual threshold & $10^{-4}$ \\
Node preprocessing & Maximum local-push operations per center & $10^6$ \\
Node preprocessing & Maximum center nodes per class & 300 \\
\hline
\end{tabular}
\end{table*}

\subsection{Cross-Domain Few-Shot Setup}

For the five LODO targets at each task level, the target dataset is excluded
from pretraining and the remaining four source datasets are used to learn the
encoder. COLORS-3 and \texttt{ogbg-molhiv} are evaluated using the graph-level
checkpoint pretrained on all five graph source domains, whereas Reddit and
\texttt{ogbn-arxiv} use the node-level checkpoint pretrained on all five node
source domains. The pretrained encoder is frozen throughout downstream
evaluation; only the task-specific downstream head is constructed from
the support set.

For every target dataset, we generate 50 deterministic episodes. Within each
episode, five support samples are drawn without replacement from every class,
followed by at most 50 query samples from the remaining examples of that
class. Thus, the support and query sets are disjoint within an episode, while
examples may be reused across different episodes. Feature standardization is
fitted exclusively on the support embeddings and then applied to the query
embeddings, preventing query statistics from entering downstream training.

For graph classification, class prototypes are computed as the means of the
support embeddings, and each query is assigned to its nearest prototype in
Euclidean distance. For node classification, a linear classifier is fitted on
the support embeddings of each episode. We report the mean and standard
deviation of Accuracy over the 50 episodes.

\subsection{Input Preprocessing}

For graph-classification datasets, intrinsic node attributes and categorical
node labels are concatenated when both are available. Datasets without
intrinsic node features use a constant scalar of one for every node. The
resulting node inputs are zero-padded to a common width of 56 dimensions.
For \texttt{ogbg-molhiv}, the encoders use the nine categorical atom-feature
fields and the unweighted molecular connectivity; categorical bond attributes
are not included in the model input.

For node classification, both source pretraining and target evaluation operate
on PPR ego-graphs~\cite{bojchevski2020ppr} centered at labeled nodes. We use approximate local-push PPR
with restart probability $\alpha=0.15$, residual threshold $10^{-4}$, and at
most $10^6$ push operations per center. Each ego-graph contains at most 400
nodes, and at most 300 center nodes are sampled from each class when building
the reusable graph cache. Because the raw feature dimensions differ across
domains, a deterministic Gaussian random projection with seed 42 maps every
node feature vector to 256 dimensions.

\subsection{Default Configuration}

Table~\ref{tab:app_default_transfer_settings} summarizes the default
settings used in the cross-domain transfer experiment.
SCGFM-ART uses a relational atlas containing $K=16$ bases, with each
base defined over a shared set of $M=32$ relational roles.
Accordingly, the atlas coordinate satisfies $q(G)\in\mathbb{R}^{K}$,
each graph--base coupling satisfies
$\widehat{T}_k\in\mathbb{R}^{N\times M}$, and the transported feature
map satisfies $H(G)\in\mathbb{R}^{M\times d_f}$.
The default values $(K,M)=(16,32)$ are used in the main comparisons
and all analyses unless otherwise stated.
In the atlas-capacity analysis, one of $K$ and $M$ is varied while the
other is fixed at its default value.
For the compared methods, we use the official or recommended
optimization settings whenever available; configurable message-passing
encoders use three layers, as stated in the main paper.

Table~\ref{tab:app_baseline_training_settings} summarizes the pretraining
configurations of the comparison methods. Here, $E$, $B$, $\eta$, and
$\lambda_{\mathrm{wd}}$ denote the number of training epochs, batch size,
learning rate, and weight decay, respectively. The configurations are fixed
across target datasets at each task level.

\begin{table*}[!t]
\centering
\caption{Pretraining configurations of the comparison methods.}
\label{tab:app_baseline_training_settings}
\setlength{\tabcolsep}{3.5pt}
\begin{tabular}{@{}lcccccccc@{}}
\hline
& \multicolumn{4}{c}{\textbf{Graph classification}}
& \multicolumn{4}{c}{\textbf{Node classification}} \\
\textbf{Method}
& $E$ & $B$ & $\eta$ & $\lambda_{\mathrm{wd}}$
& $E$ & $B$ & $\eta$ & $\lambda_{\mathrm{wd}}$ \\
\hline
GCN / GAT / GIN
& 100 & 512 & $1\times10^{-3}$ & 0
& 100 & 256 & $1\times10^{-3}$ & 0 \\

GraphCL
& 100 & 128 & $1\times10^{-3}$ & 0
& 100 & 128 & $1\times10^{-3}$ & 0 \\

GraphMAE
& 100 & 128 & $1\times10^{-3}$ & $5\times10^{-4}$
& 100 & 128 & $1\times10^{-3}$ & $5\times10^{-4}$ \\

GraphACL
& 100 & 128 & $1\times10^{-3}$ & $1\times10^{-6}$
& 100 & 128 & $1\times10^{-3}$ & $5\times10^{-4}$ \\

GraphGlue
& 100 & 128 & $1\times10^{-3}$ & $5\times10^{-4}$
& 100 & 128 & $1\times10^{-3}$ & $5\times10^{-4}$ \\

GCOPE
& 100 & 128 & $1\times10^{-3}$ & $1\times10^{-5}$
& 100 & 128 & $1\times10^{-3}$ & $5\times10^{-4}$ \\

GiT
& 100 & 128 & $1\times10^{-3}$ & $1\times10^{-8}$
& 100 & 128 & $1\times10^{-3}$ & $5\times10^{-4}$ \\

MDGMIX
& 100 & 128 & $1\times10^{-3}$ & $1\times10^{-6}$
& 100 & 128 & $1\times10^{-3}$ & $5\times10^{-4}$ \\

RiemannGFM
& 100 & 64 & $1\times10^{-3}$ & $1\times10^{-6}$
& 100 & 128 & $1\times10^{-3}$ & $5\times10^{-4}$ \\

SCGFM
& 100 & 128 & $2\times10^{-2}$ & 0
& 60 & 64 & $1\times10^{-2}$ & 0 \\
\hline
\end{tabular}
\end{table*}

For graph-level pretraining, GCN, GAT, GIN, and GiT are optimized using
AdamW~\cite{loshchilov2019adamw}, while the remaining comparison methods use Adam~\cite{kingma2015adam}. For node-level
pretraining, GraphCL and SCGFM use Adam, whereas the other comparison
methods use AdamW. No learning-rate scheduler is employed.

The graph-level encoders use a hidden dimension of 64, with GCN, GAT, and
GIN producing 192-dimensional graph representations. At the node level,
GCN, GAT, GIN, GraphMAE, GraphACL, GraphGlue, GCOPE, GiT, and MDGMIX use
a hidden dimension of 128, whereas GraphCL and RiemannGFM use a hidden
dimension of 64.

We use $step=20$ Sinkhorn iterations throughout all experiments.
In practice, this is sufficient for stable marginal projection, with the maximum marginal residual typically below $10^{-7}$ during training.

\subsection{Implementation and Hardware}

The experiments are implemented in Python using PyTorch and PyTorch Geometric
(PyG), with CUDA acceleration for model training and inference. Unless
otherwise stated, each experiment is executed on a single NVIDIA A800 GPU with
80~GB of memory. The host machine is equipped with an Intel(R) Xeon(R) Gold
6348 CPU running at 2.60~GHz.

\section{Supplementary Theoretical Analysis}
\label{app:supp_theory}

\subsection{Relational Structures and Quotient Geometry}
\label{app:supp_relational_geometry}

\subsubsection{Basic definitions}

\begin{definition}[Finite measured relational structure]
A finite measured relational structure is a triplet
\begin{equation}
G=(V,A,\mu_G),
\end{equation}
where $V$ is a finite set,
$A:V\times V\rightarrow[0,1]$ is symmetric with $A(i,i)=0$, and
$\mu_G$ is a probability measure on $V$.
\end{definition}
For two measured relational structures
$G=(V,A,\mu_G)$ and $G'=(V',A',\mu_{G'})$, define the transport polytope
\begin{equation}
\Pi(\mu_G,\mu_{G'})
=
\left\{
T\in\mathbb R_+^{|V|\times|V'|}:
T\mathbf 1=\mu_G,
\;T^\top\mathbf 1=\mu_{G'}
\right\}.
\label{eq:supp_transport_polytope}
\end{equation}
For any $T\in\Pi(\mu_G,\mu_{G'})$, the squared relational distortion is
\begin{equation}
\mathcal E(A,A';T)
=
\sum_{i,j,a,b}
(A_{ij}-A'_{ab})^2T_{ia}T_{jb}.
\label{eq:supp_relational_energy}
\end{equation}
The associated relational discrepancy is
\begin{equation}
d(G,G')
=
\sqrt{
\min_{T\in\Pi(\mu_G,\mu_{G'})}
\mathcal E(A,A';T)
}.
\label{eq:supp_relational_discrepancy}
\end{equation}

\begin{proposition}[Relational graph geometry]
\label{prop:supp_relational_geometry}
The discrepancy $d$ is finite, symmetric, and invariant to node relabeling.
Modulo zero-discrepancy equivalence, it induces a metric on a quotient space
$\mathcal X$ containing both observed graphs and relational bases.
\end{proposition}

\begin{proof}
We first establish boundedness. Because
$A_{ij},A'_{ab}\in[0,1]$, for every feasible $T$,
\begin{align}
0
\le
\mathcal E(A,A';T)
&=
\sum_{i,j,a,b}
(A_{ij}-A'_{ab})^2T_{ia}T_{jb}
\nonumber\\
&\le
\sum_{i,j,a,b}T_{ia}T_{jb}.
\end{align}
Since $T$ has unit total mass,
\begin{equation}
\sum_{i,a}T_{ia}=1,
\end{equation}
and therefore
\begin{equation}
\sum_{i,j,a,b}T_{ia}T_{jb}
=
\left(\sum_{i,a}T_{ia}\right)^2
=1.
\end{equation}
Hence
\begin{equation}
0\le d(G,G')\le 1.
\label{eq:supp_d_bounded}
\end{equation}

Symmetry follows by transposing the coupling. Specifically,
$T\in\Pi(\mu_G,\mu_{G'})$ if and only if
$T^\top\in\Pi(\mu_{G'},\mu_G)$, and direct reindexing gives
\begin{equation}
\mathcal E(A,A';T)
=
\mathcal E(A',A;T^\top).
\end{equation}
Taking minima over the corresponding feasible sets yields
$d(G,G')=d(G',G)$.

We next prove invariance to node relabeling. Let $P$ and $P'$ be
permutation matrices acting on $V$ and $V'$, respectively, and define
\begin{equation}
\widetilde A=PAP^\top,
\qquad
\widetilde A'=P'A'P'^\top.
\end{equation}
The map
\begin{equation}
T\longmapsto \widetilde T=PTP'^\top
\end{equation}
is a bijection from
$\Pi(\mu_G,\mu_{G'})$ to
$\Pi(P\mu_G,P'\mu_{G'})$.
Moreover, reindexing the summation in
Eq.~\eqref{eq:supp_relational_energy} gives
\begin{equation}
\mathcal E(\widetilde A,\widetilde A';\widetilde T)
=
\mathcal E(A,A';T).
\end{equation}
Taking minima therefore preserves the discrepancy.

The remaining metric structure follows from the standard order-$2$ network Gromov--Wasserstein result for finite measure networks
\cite{chowdhury2019networkgw}. In particular, the discrepancy satisfies the triangle inequality and vanishes exactly on weak-isomorphism classes.
Define
\begin{equation}
G\sim_d G'
\quad\Longleftrightarrow\quad
 d(G,G')=0
\end{equation}
and
\begin{equation}
\mathcal X=\mathfrak G/\!\sim_d.
\end{equation}
If $G\sim_d\widetilde G$ and
$G'\sim_d\widetilde G'$, the triangle inequality gives
\begin{equation}
\left|
 d(G,G')-d(\widetilde G,\widetilde G')
\right|
\le
 d(G,\widetilde G)+d(G',\widetilde G')=0.
\end{equation}
Thus the distance between equivalence classes is independent of the
representatives. After quotienting out zero-discrepancy pairs, identity of
indiscernibles holds and the other metric axioms are inherited from $d$.
\end{proof}

\subsubsection{Finite relational atlas}

Let $\mathcal X_0=\operatorname{supp}(\mathbb P)\subseteq\mathcal X$ denote
the support of the task distribution. We assume that
$(\mathcal X_0,d)$ is totally bounded. Thus, for every $\epsilon>0$ there
exists a finite collection of relational structures whose $\epsilon$-balls
cover $\mathcal X_0$.

\begin{definition}[Finite relational atlas]
A $K$-base relational atlas is
\begin{equation}
\mathcal A=\{\beta_1,\ldots,\beta_K\}\subseteq\mathcal X.
\end{equation}
Its covering radius over $\mathcal X_0$ is
\begin{equation}
\epsilon_{\mathcal A}
:=
\sup_{G\in\mathcal X_0}
\min_{1\le k\le K}d(G,\beta_k),
\label{eq:supp_atlas_radius}
\end{equation}
and its ideal relational coordinate is
\begin{equation}
q_{\mathcal A}^*(G)
:=
\bigl(d(G,\beta_k)\bigr)_{k=1}^K.
\label{eq:supp_ideal_coordinate}
\end{equation}
\end{definition}

Total boundedness guarantees that finite relational atlases exist at every
positive resolution. The following result quantifies the information retained
by a realized atlas.

\begin{theorem}[Finite-atlas coordinate fidelity]
\label{thm:supp_atlas_fidelity}
Let $\mathcal A$ be a finite relational atlas with covering radius
$\epsilon_{\mathcal A}$ over $\mathcal X_0$. Then, for any
$G,G'\in\mathcal X_0$,
\begin{equation}
 d(G,G')-2\epsilon_{\mathcal A}
\le
\left\|
q_{\mathcal A}^*(G)-q_{\mathcal A}^*(G')
\right\|_\infty
\le
 d(G,G').
\label{eq:supp_atlas_fidelity}
\end{equation}
\end{theorem}

\begin{proof}
For any atlas element $\beta_k$, the triangle inequality gives
\begin{equation}
 d(G,\beta_k)
\le
 d(G,G')+d(G',\beta_k).
\end{equation}
Exchanging $G$ and $G'$ yields
\begin{equation}
\left|
 d(G,\beta_k)-d(G',\beta_k)
\right|
\le
 d(G,G').
\end{equation}
Taking the maximum over $k$ gives
\begin{equation}
\left\|
q_{\mathcal A}^*(G)-q_{\mathcal A}^*(G')
\right\|_\infty
\le
 d(G,G').
\label{eq:supp_atlas_upper}
\end{equation}

For the lower bound, choose an index $r$ such that
\begin{equation}
 d(G,\beta_r)
=
\min_k d(G,\beta_k)
\le
\epsilon_{\mathcal A}.
\end{equation}
Again by the triangle inequality,
\begin{equation}
 d(G',\beta_r)
\ge
 d(G,G')-d(G,\beta_r)
\ge
 d(G,G')-\epsilon_{\mathcal A}.
\end{equation}
Therefore,
\begin{align}
\left|
 d(G',\beta_r)-d(G,\beta_r)
\right|
&\ge
 d(G',\beta_r)-d(G,\beta_r)
\nonumber\\
&\ge
 d(G,G')-2\epsilon_{\mathcal A}.
\end{align}
Since the $\ell_\infty$ norm is at least the magnitude of any individual
coordinate,
\begin{equation}
\left\|
q_{\mathcal A}^*(G)-q_{\mathcal A}^*(G')
\right\|_\infty
\ge
 d(G,G')-2\epsilon_{\mathcal A}.
\label{eq:supp_atlas_lower}
\end{equation}
Combining Eqs.~\eqref{eq:supp_atlas_upper} and
\eqref{eq:supp_atlas_lower} proves the theorem.
\end{proof}

\subsection{Theoretical Properties of Amortized Relational Transport}
\label{app:supp_art_properties}

For the $k$-th relational base
$\beta_k=([M],B_k,\nu)$, ART produces proposal logits
$L_k\in\mathbb R^{N\times M}$ and applies Sinkhorn normalization with
marginals $\mu_G$ and $\nu$. The theoretical operator below denotes the
converged Sinkhorn projection; a fixed-iteration implementation can be
quantified by its marginal residual.

\begin{lemma}[Feasibility and permutation equivariance of ART]
\label{lem:supp_art_equivariance}
Assume that the proposal kernel and both marginals are strictly positive.
Then the Sinkhorn limit $\widehat T_k$ satisfies
\begin{equation}
\widehat T_k\mathbf 1_M=\mu_G,
\qquad
\widehat T_k^\top\mathbf 1_N=\nu.
\label{eq:supp_sinkhorn_marginals}
\end{equation}
Moreover, for any permutation matrix $P$ acting on the input nodes,
\begin{equation}
\widehat T_k(PAP^\top,P\mu_G)
=
P\widehat T_k(A,\mu_G).
\label{eq:supp_art_equivariance}
\end{equation}
Consequently, the relational energy induced by ART is invariant to input-node
relabeling.
\end{lemma}

\begin{proof}
Let the strictly positive Sinkhorn kernel be
\begin{equation}
\Xi_{ia}
=
\exp\left(\frac{L_k(i,a)}{\tau_{\mathrm{sk}}}\right)>0.
\end{equation}
Sinkhorn scaling constructs matrices of the form
\begin{equation}
T=\operatorname{diag}(u)\Xi\operatorname{diag}(v),
\end{equation}
with alternating updates
\begin{equation}
u\leftarrow\mu_G\oslash(\Xi v),
\qquad
v\leftarrow\nu\oslash(\Xi^\top u),
\end{equation}
where $\oslash$ denotes element-wise division. For a strictly positive
kernel and strictly positive marginals of equal total mass, the classical
Sinkhorn scaling theorem gives convergence to a matrix satisfying the desired
row and column marginals. Hence
Eq.~\eqref{eq:supp_sinkhorn_marginals} holds and
$\widehat T_k\in\Pi(\mu_G,\nu)$.

We next establish equivariance. Under the node relabeling
\begin{equation}
A'=PAP^\top,
\qquad
\mu_G'=P\mu_G,
\end{equation}
the normalized degree feature transforms as $X_s'=PX_s$. The shared GNN encoder is permutation equivariant, so
\begin{equation}
U'=PU.
\end{equation}
The base-role embeddings are unchanged because the permutation acts only on
the input graph. Therefore, the compatibility logits and positive kernel obey
\begin{equation}
L_k'=PL_k,
\qquad
\Xi_k'=P\Xi_k.
\end{equation}
Sinkhorn normalization consists only of graph-side row scaling and base-side
column scaling. Consequently, every graph-side scaling vector is permuted by
$P$, while the base-side scaling vector is unchanged. Thus
\begin{equation}
\widehat T_k'=P\widehat T_k,
\end{equation}
which proves Eq.~\eqref{eq:supp_art_equivariance}.

Finally, substituting
$A'=PAP^\top$ and
$\widehat T_k'=P\widehat T_k$ into
Eq.~\eqref{eq:supp_relational_energy}, followed by reindexing of the graph
nodes, gives
\begin{equation}
\mathcal E(A',B_k;\widehat T_k')
=
\mathcal E(A,B_k;\widehat T_k).
\end{equation}
Hence the graph-level ART energy is invariant to node relabeling.
\end{proof}

For a finite number of Sinkhorn iterations, feasibility can be monitored by
\begin{equation}
\operatorname{Res}_{\mathrm{marg}}(T)
=
\max\left\{
\|T\mathbf 1_M-\mu_G\|_\infty,
\|T^\top\mathbf 1_N-\nu\|_\infty
\right\}.
\label{eq:supp_marginal_residual}
\end{equation}

\subsubsection{Fast relational-energy identity}

\begin{proposition}[Fast relational-energy identity]
\label{prop:supp_fast_energy}
For any feasible coupling $T\in\Pi(\mu_G,\nu)$ and symmetric relations
$A$ and $B$,
\begin{equation}
\mathcal E(A,B;T)
=
\chi(A)+\chi(B)-2\langle AT,TB\rangle_F,
\label{eq:supp_fast_energy}
\end{equation}
where
\begin{equation}
\chi(A)=\sum_{i,j}A_{ij}^2\mu_G(i)\mu_G(j),
\qquad
\chi(B)=\sum_{a,b}B_{ab}^2\nu_a\nu_b.
\end{equation}
\end{proposition}

\begin{proof}
Expanding the square gives
\begin{equation}
(A_{ij}-B_{ab})^2
=
A_{ij}^2+B_{ab}^2-2A_{ij}B_{ab}.
\end{equation}
For the first term, feasibility gives
\begin{align}
\sum_{i,j,a,b}A_{ij}^2T_{ia}T_{jb}
&=
\sum_{i,j}A_{ij}^2
\left(\sum_aT_{ia}\right)
\left(\sum_bT_{jb}\right)
\nonumber\\
&=
\sum_{i,j}A_{ij}^2\mu_G(i)\mu_G(j)
=
\chi(A).
\end{align}
Similarly,
\begin{align}
\sum_{i,j,a,b}B_{ab}^2T_{ia}T_{jb}
&=
\sum_{a,b}B_{ab}^2
\left(\sum_iT_{ia}\right)
\left(\sum_jT_{jb}\right)
\nonumber\\
&=
\sum_{a,b}B_{ab}^2\nu_a\nu_b
=
\chi(B).
\end{align}
For the cross term, symmetry of $A$ and $B$ yields
\begin{align}
\sum_{i,j,a,b}A_{ij}B_{ab}T_{ia}T_{jb}
&=
\operatorname{tr}(T^\top A T B)
\nonumber\\
&=
\langle AT,TB\rangle_F.
\end{align}
Combining the three terms proves
Eq.~\eqref{eq:supp_fast_energy}.
\end{proof}

For a loop-free binary adjacency relation,
$A_{ij}^2=A_{ij}$ and the graph-side quantities admit the exact sparse forms
\begin{equation}
\chi(A)
=
\sum_{(i,j)\in E}\mu_G(i)\mu_G(j),
\label{eq:supp_sparse_chi}
\end{equation}
and
\begin{equation}
(AT)(i,a)
=
\sum_{j:(i,j)\in E}T(j,a).
\label{eq:supp_sparse_AT}
\end{equation}
Thus, the energy can be evaluated without materializing an $N\times N$
dense relation matrix.

\subsubsection{Amortized atlas coordinate fidelity}

For each base, define the ideal and amortized energies
\begin{align}
e_k^*(G)
&:=
\min_{T\in\Pi(\mu_G,\nu)}
\mathcal E(A,B_k;T)
=d^2(G,\beta_k),
\\
e_k(G)
&:=
\mathcal E(A,B_k;\widehat T_k)
\ge e_k^*(G),
\qquad
q_k(G):=\sqrt{e_k(G)}.
\end{align}
Let
\begin{equation}
q(G):=(q_k(G))_{k=1}^K,
\qquad
\eta(G)
:=
\|q(G)-q_{\mathcal A}^*(G)\|_\infty.
\label{eq:supp_eta}
\end{equation}

\begin{corollary}[Amortized atlas coordinate fidelity]
\label{cor:supp_amortized_coordinate}
Let $\mathcal A$ be a finite relational atlas with covering radius
$\epsilon_{\mathcal A}$ over $\mathcal X_0$. For any
$G,G'\in\mathcal X_0$,
\begin{equation}
\begin{aligned}
 d(G,G')
-2\epsilon_{\mathcal A}
-\eta(G)-\eta(G')
&\le
\|q(G)-q(G')\|_\infty
\\
&\le
 d(G,G')+\eta(G)+\eta(G').
\end{aligned}
\label{eq:supp_amortized_coordinate_bound}
\end{equation}
\end{corollary}

\begin{proof}
By the triangle inequality,
\begin{align}
\|q(G)-q(G')\|_\infty
&\le
\|q(G)-q_{\mathcal A}^*(G)\|_\infty
+
\|q_{\mathcal A}^*(G)-q_{\mathcal A}^*(G')\|_\infty
\nonumber\\
&\quad+
\|q_{\mathcal A}^*(G')-q(G')\|_\infty.
\end{align}
By the definition of $\eta(\cdot)$ and
Theorem~\ref{thm:supp_atlas_fidelity},
\begin{equation}
\|q(G)-q(G')\|_\infty
\le
 d(G,G')+\eta(G)+\eta(G').
\label{eq:supp_amortized_upper}
\end{equation}

Conversely, applying the triangle inequality to
$q_{\mathcal A}^*(G)-q_{\mathcal A}^*(G')$ gives
\begin{align}
\|q(G)-q(G')\|_\infty
&\ge
\|q_{\mathcal A}^*(G)-q_{\mathcal A}^*(G')\|_\infty
\nonumber\\
&\quad-
\|q(G)-q_{\mathcal A}^*(G)\|_\infty
-
\|q(G')-q_{\mathcal A}^*(G')\|_\infty.
\end{align}
Again using Theorem~\ref{thm:supp_atlas_fidelity},
\begin{equation}
\|q(G)-q(G')\|_\infty
\ge
 d(G,G')
-2\epsilon_{\mathcal A}
-\eta(G)
-\eta(G').
\label{eq:supp_amortized_lower}
\end{equation}
Combining Eqs.~\eqref{eq:supp_amortized_upper} and
\eqref{eq:supp_amortized_lower} proves the result.
\end{proof}

The bound separates two sources of distortion: the finite-atlas resolution
$\epsilon_{\mathcal A}$ and the coordinate approximation error $\eta(G)$
introduced by amortized transport.

\subsection{Amortized Relational Coverage}
\label{app:supp_coverage}

For a graph $G$ and atlas $\mathcal A$, the normalized soft-min coverage is
\begin{equation}
\ell_{\mathrm{mean}}(G)
=
-\tau_m\log\left[
\frac1K\sum_{k=1}^K
\exp\left(-\frac{e_k(G)}{\tau_m}\right)
\right].
\label{eq:supp_mean_coverage}
\end{equation}
We first record the elementary soft-min bound used by the main theorem.

\begin{lemma}[Normalized soft-min bound]
\label{lem:supp_softmin}
For arbitrary nonnegative energies $e_1,\ldots,e_K$ and $\tau_m>0$,
\begin{equation}
\min_k e_k
\le
-\tau_m\log\left[
\frac1K\sum_{k=1}^K e^{-e_k/\tau_m}
\right]
\le
\min_k e_k+\tau_m\log K.
\label{eq:supp_softmin_bound}
\end{equation}
\end{lemma}

\begin{proof}
Let $m=\min_k e_k$. Since $e_k\ge m$,
\begin{equation}
\frac1K\sum_{k=1}^K e^{-e_k/\tau_m}
\le
 e^{-m/\tau_m}.
\end{equation}
Applying $-\tau_m\log(\cdot)$ gives the lower bound. If $j$ satisfies
$e_j=m$, then
\begin{equation}
\frac1K\sum_{k=1}^K e^{-e_k/\tau_m}
\ge
\frac1K e^{-m/\tau_m},
\end{equation}
which yields the upper bound after applying $-\tau_m\log(\cdot)$.
\end{proof}

Define the ideal and amortized nearest-base energies by
\begin{equation}
m^*(G):=\min_k e_k^*(G),
\qquad
\widehat m(G):=\min_k e_k(G).
\label{eq:supp_nearest_energies}
\end{equation}
Let
\begin{equation}
\mathcal K^*(G):=\arg\min_k e_k^*(G)
\end{equation}
and define the ART excess on the ideal nearest bases as
\begin{equation}
\Delta(G)
:=
\min_{k\in\mathcal K^*(G)}
\left[e_k(G)-e_k^*(G)\right]
\ge0.
\label{eq:supp_delta}
\end{equation}

\begin{theorem}[Amortized coverage bound]
\label{thm:supp_amortized_coverage}
For every graph $G$,
\begin{equation}
m^*(G)
\le
\ell_{\mathrm{mean}}(G)
\le
m^*(G)+\Delta(G)+\tau_m\log K.
\label{eq:supp_coverage_bound}
\end{equation}
\end{theorem}

\begin{proof}
Because $\widehat T_k\in\Pi(\mu_G,\nu)$ and $e_k^*(G)$ is the minimum of
$\mathcal E(A,B_k;T)$ over the same feasible set,
\begin{equation}
e_k(G)-e_k^*(G)\ge0
\qquad\text{for every }k.
\end{equation}
Taking minima over $k$ gives
\begin{equation}
m^*(G)\le\widehat m(G).
\label{eq:supp_mstar_le_mhat}
\end{equation}

Choose
\begin{equation}
k^\dagger
\in
\arg\min_{k\in\mathcal K^*(G)}
\left[e_k(G)-e_k^*(G)\right].
\end{equation}
By definition,
\begin{equation}
e_{k^\dagger}^*(G)=m^*(G),
\qquad
e_{k^\dagger}(G)-e_{k^\dagger}^*(G)=\Delta(G).
\end{equation}
Therefore,
\begin{align}
\widehat m(G)
&=\min_k e_k(G)
\nonumber\\
&\le e_{k^\dagger}(G)
\nonumber\\
&=m^*(G)+\Delta(G).
\label{eq:supp_mhat_upper}
\end{align}
Applying Lemma~\ref{lem:supp_softmin} to the amortized energies gives
\begin{equation}
\widehat m(G)
\le
\ell_{\mathrm{mean}}(G)
\le
\widehat m(G)+\tau_m\log K.
\end{equation}
Combining this inequality with
Eqs.~\eqref{eq:supp_mstar_le_mhat} and
\eqref{eq:supp_mhat_upper} proves
Eq.~\eqref{eq:supp_coverage_bound}.
\end{proof}

The quantities $\eta(G)$ and $\Delta(G)$ describe complementary effects of
amortization. The former controls the largest coordinate discrepancy over the
entire atlas and therefore appears in pairwise representation fidelity. The
latter only measures excess energy on an ideal nearest base, which is
sufficient for analyzing the soft-min coverage objective.

\begin{corollary}[Finite-atlas coverage consistency]
\label{cor:supp_coverage_consistency}
Let $\epsilon_{\mathcal A}$ be the covering radius of $\mathcal A$ over
$\mathcal X_0$. For every $G\in\mathcal X_0$,
\begin{equation}
0
\le
\ell_{\mathrm{mean}}(G)
\le
\epsilon_{\mathcal A}^2
+
\Delta(G)
+
\tau_m\log K.
\label{eq:supp_coverage_consistency}
\end{equation}
\end{corollary}

\begin{proof}
By the definition of the covering radius, for every
$G\in\mathcal X_0$ there exists an atlas element $\beta_j$ such that
\begin{equation}
d(G,\beta_j)\le\epsilon_{\mathcal A}.
\end{equation}
Therefore,
\begin{align}
m^*(G)
&=\min_k d^2(G,\beta_k)
\nonumber\\
&\le d^2(G,\beta_j)
\le\epsilon_{\mathcal A}^2.
\end{align}
Substituting this inequality into the upper bound of
Theorem~\ref{thm:supp_amortized_coverage} gives
\begin{equation}
\ell_{\mathrm{mean}}(G)
\le
\epsilon_{\mathcal A}^2+\Delta(G)+\tau_m\log K.
\end{equation}
Moreover, $e_k(G)\ge0$ implies
$\exp(-e_k(G)/\tau_m)\le1$. Hence the normalized average inside the
logarithm in Eq.~\eqref{eq:supp_mean_coverage} is at most one and
$\ell_{\mathrm{mean}}(G)\ge0$. Combining the two inequalities proves the
corollary.
\end{proof}

\subsubsection{Population-risk implication}

Define the ideal atlas risk and its trainable surrogate by
\begin{equation}
\mathcal R^*(\mathcal A)
:=
\mathbb E_{G\sim\mathbb P}[m^*(G)],
\qquad
\widehat{\mathcal R}(\mathcal A,\theta)
:=
\mathbb E_{G\sim\mathbb P}[\ell_{\mathrm{mean}}(G)].
\end{equation}
Taking expectations in
Theorem~\ref{thm:supp_amortized_coverage} yields
\begin{equation}
\mathcal R^*(\mathcal A)
\le
\widehat{\mathcal R}(\mathcal A,\theta)
\le
\mathcal R^*(\mathcal A)
+
\mathbb E_G[\Delta(G)]
+
\tau_m\log K.
\label{eq:supp_population_bound}
\end{equation}
Thus, mean relational coverage is a tractable surrogate upper bound on the
ideal finite-atlas risk. Its excess is determined by the amortized alignment
error and the soft-min relaxation.

\subsubsection{Variational interpretation}

Let
\begin{equation}
\Delta_K
=
\left\{
\rho\in\mathbb R_+^K:
\sum_{k=1}^K\rho_k=1
\right\}
\end{equation}
be the probability simplex. Define
\begin{equation}
\Phi(\rho)
=
\sum_{k=1}^K\rho_ke_k
+
\tau_m\sum_{k=1}^K\rho_k\log(K\rho_k).
\label{eq:supp_variational_objective}
\end{equation}

\begin{proposition}[Variational form of mean coverage]
\label{prop:supp_variational_coverage}
The normalized soft-min objective satisfies
\begin{equation}
\ell_{\mathrm{mean}}
=
\min_{\rho\in\Delta_K}\Phi(\rho),
\end{equation}
with minimizer
\begin{equation}
\rho_k^*
=
\frac{\exp(-e_k/\tau_m)}
{\sum_{\ell=1}^K\exp(-e_\ell/\tau_m)}.
\label{eq:supp_variational_responsibility}
\end{equation}
\end{proposition}

\begin{proof}
Introduce a Lagrange multiplier $\alpha$ for
$\sum_k\rho_k=1$. The stationarity condition is
\begin{equation}
e_k
+
\tau_m\left[\log(K\rho_k)+1\right]
+
\alpha
=0.
\end{equation}
Therefore,
\begin{equation}
\rho_k\propto\exp(-e_k/\tau_m),
\end{equation}
and normalization gives
Eq.~\eqref{eq:supp_variational_responsibility}. Substituting the minimizer
into Eq.~\eqref{eq:supp_variational_objective} yields
\begin{align}
\min_{\rho\in\Delta_K}\Phi(\rho)
&=
-\tau_m\log\left[
\frac1K\sum_{k=1}^K\exp(-e_k/\tau_m)
\right]
\nonumber\\
&=\ell_{\mathrm{mean}},
\end{align}
which proves the proposition.
\end{proof}

\subsection{Permutation Invariance of the ART-Full Representation}
\label{app:supp_art_full_invariance}

At transfer time, SCGFM-ART forms energy-based responsibilities
\begin{equation}
w_k(G)
=
\frac{\exp(-e_k(G)/\tau_E)}
{\sum_{\ell=1}^K\exp(-e_\ell(G)/\tau_E)},
\end{equation}
and the coupling mixture
\begin{equation}
T_{\mathrm{mix}}(G)
=
\sum_{k=1}^K w_k(G)\widehat T_k.
\label{eq:supp_tmix}
\end{equation}
The transported feature map and final representation are
\begin{equation}
H(G)=NT_{\mathrm{mix}}(G)^\top X,
\end{equation}
and
\begin{equation}
z(G)
=
\left[q(G)\|\operatorname{vec}(H(G))\right].
\end{equation}

\begin{proposition}[Permutation invariance of ART-full]
\label{prop:supp_art_full_invariance}
Let $P$ be an $N\times N$ permutation matrix and define the relabeled input
by
\begin{equation}
A'=PAP^\top,
\qquad
X'=PX,
\qquad
\mu_G'=P\mu_G.
\end{equation}
Then
\begin{equation}
q(G')=q(G),
\qquad
H(G')=H(G),
\qquad
z(G')=z(G).
\end{equation}
\end{proposition}

\begin{proof}
By Lemma~\ref{lem:supp_art_equivariance},
\begin{equation}
\widehat T_k'=P\widehat T_k.
\end{equation}
The same lemma also gives energy invariance,
\begin{equation}
e_k(G')=e_k(G)
\qquad\text{for every }k.
\end{equation}
Hence
\begin{equation}
q(G')=q(G)
\end{equation}
and
\begin{equation}
w_k(G')=w_k(G).
\end{equation}
Using Eq.~\eqref{eq:supp_tmix},
\begin{align}
T_{\mathrm{mix}}(G')
&=
\sum_{k=1}^K w_k(G')\widehat T_k'
\nonumber\\
&=
\sum_{k=1}^K w_k(G)P\widehat T_k
\nonumber\\
&=
PT_{\mathrm{mix}}(G).
\end{align}
Therefore,
\begin{align}
H(G')
&=
N T_{\mathrm{mix}}(G')^\top X'
\nonumber\\
&=
N T_{\mathrm{mix}}(G)^\top P^\top PX
\nonumber\\
&=
N T_{\mathrm{mix}}(G)^\top X
=H(G).
\end{align}
Combining $q(G')=q(G)$ and $H(G')=H(G)$ yields
\begin{equation}
z(G')=z(G).
\end{equation}
\end{proof}

\subsection{Sparse Evaluation and Complexity Details}
\label{app:supp_complexity}

For completeness, we detail the complexity stated in the main paper. The
sparse graph encoder requires $O(N+|E|)$ operations for fixed hidden width
and depth. Encoding all relational bases costs $O(KM^2)$. Constructing
$K$ graph--base compatibility matrices costs $O(KNM)$, and $S$ Sinkhorn
iterations cost $O(SKNM)$. By
Eqs.~\eqref{eq:supp_sparse_chi}--\eqref{eq:supp_sparse_AT}, evaluating
$AT_k$ over all bases costs $O(K|E|M)$. The multiplication
$T_kB_k$ costs $O(KNM^2)$, after which the Frobenius inner product is
$O(KNM)$. Hence the complete leading forward complexity is
\begin{equation}
O\left(
N+|E|
+KM^2
+SKNM
+K|E|M
+KNM^2
\right).
\label{eq:supp_time_complexity}
\end{equation}
For fixed architectural parameters $K$, $M$, $S$, and hidden width, this is
linear in the sparse graph representation,
\begin{equation}
O(N+|E|).
\end{equation}

The principal stored quantities are the sparse graph, graph-node embeddings,
$K$ graph--base couplings, and $K$ relational bases. The resulting per-graph
memory complexity is
\begin{equation}
O\left(N+|E|+KNM+KM^2\right).
\label{eq:supp_memory_complexity}
\end{equation}
No term requires an $N\times N$ dense relation tensor.

\subsection{Additional Fixed-Coupling Identity}
\label{app:supp_fixed_coupling}

This final identity is not required by the learning objective or the
transfer representation, but it is useful for implementation verification.
For a feasible coupling $T\in\Pi(\mu_G,\nu)$, let
\begin{equation}
D_\nu=\operatorname{diag}(\nu),
\qquad
\Lambda_T=D_\nu^{-1}T^\top,
\end{equation}
and define the relation transported into the common base-role space as
\begin{equation}
\overline A_T
=
\Lambda_TA\Lambda_T^\top
=
D_\nu^{-1}T^\top A T D_\nu^{-1}.
\label{eq:supp_transported_relation}
\end{equation}
For $C,D\in\mathbb R^{M\times M}$, define
\begin{equation}
\langle C,D\rangle_\nu
=
\sum_{a,b}C_{ab}D_{ab}\nu_a\nu_b,
\qquad
\|C\|_\nu^2=\langle C,C\rangle_\nu.
\end{equation}

\begin{proposition}[Fixed-coupling energy decomposition]
\label{prop:supp_fixed_coupling}
For every feasible $T$,
\begin{equation}
\mathcal E(A,B;T)
=
\|B-\overline A_T\|_\nu^2
+
\mathcal V(A;T),
\label{eq:supp_fixed_coupling_decomp}
\end{equation}
where
\begin{equation}
\mathcal V(A;T)
=
\sum_{i,j}A_{ij}^2\mu_G(i)\mu_G(j)
-
\|\overline A_T\|_\nu^2
\ge0.
\label{eq:supp_transport_variance}
\end{equation}
\end{proposition}

\begin{proof}
From Eq.~\eqref{eq:supp_transported_relation},
\begin{align}
\langle B,\overline A_T\rangle_\nu
&=
\sum_{a,b}B_{ab}
\left[
\frac{\sum_{i,j}T_{ia}A_{ij}T_{jb}}{\nu_a\nu_b}
\right]
\nu_a\nu_b
\nonumber\\
&=
\sum_{i,j,a,b}A_{ij}B_{ab}T_{ia}T_{jb}.
\end{align}
Using Proposition~\ref{prop:supp_fast_energy},
\begin{align}
\mathcal E(A,B;T)
&=
\sum_{i,j}A_{ij}^2\mu_G(i)\mu_G(j)
+
\|B\|_\nu^2
-2\langle B,\overline A_T\rangle_\nu
\nonumber\\
&=
\|B-\overline A_T\|_\nu^2
+
\mathcal V(A;T),
\end{align}
which proves the decomposition.

It remains to show nonnegativity. Since
\begin{equation}
\sum_i \frac{T_{ia}}{\nu_a}=1,
\end{equation}
the coefficients $T_{ia}/\nu_a$ define a probability distribution over
input nodes conditional on role $a$. Jensen's inequality gives, for each
$(a,b)$,
\begin{equation}
\overline A_T(a,b)^2
\le
\sum_{i,j}A_{ij}^2
\frac{T_{ia}}{\nu_a}
\frac{T_{jb}}{\nu_b}.
\end{equation}
Multiplying by $\nu_a\nu_b$, summing over $(a,b)$, and using
$T\mathbf 1_M=\mu_G$ yields
\begin{equation}
\|\overline A_T\|_\nu^2
\le
\sum_{i,j}A_{ij}^2\mu_G(i)\mu_G(j).
\end{equation}
Therefore $\mathcal V(A;T)\ge0$.
\end{proof}

\end{document}